\documentclass[12pt]{article} 

\usepackage{amsmath,amsfonts,bm}

\def\eqref#1{equation~\ref{#1}}

\def\1{\bm{1}}

\def\rvx{{\mathbf{x}}}

\def\rvz{{\mathbf{z}}}

\def\vmu{{\bm{\mu}}}
\def\vtheta{{\bm{\theta}}}

\def\ve{{\bm{e}}}

\def\vu{{\bm{u}}}
\def\vv{{\bm{v}}}

\def\vx{{\bm{x}}}
\def\vy{{\bm{y}}}
\def\vz{{\bm{z}}}

\def\mI{{\bm{I}}}

\def\mL{{\bm{L}}}

\def\mSigma{{\bm{\Sigma}}}

\DeclareMathAlphabet{\mathsfit}{\encodingdefault}{\sfdefault}{m}{sl}
\SetMathAlphabet{\mathsfit}{bold}{\encodingdefault}{\sfdefault}{bx}{n}

\def\sB{{\mathbb{B}}}

\newcommand{\E}{\mathbb{E}}

\newcommand{\R}{\mathbb{R}}

\usepackage[letterpaper,left=1.0in,right=1.0in,top=1.0in,bottom=1.0in]{geometry}
\usepackage{hyperref}
\usepackage{url}
\usepackage{wrapfig}
\usepackage{float}
\usepackage{multirow} 
\usepackage{subcaption}
\usepackage{amsthm}
\newtheorem{theorem}{Theorem}
\usepackage{graphicx}
\usepackage{booktabs}
\usepackage{tabularx}
\usepackage{array}
\usepackage{natbib} 

\title{BayesNDE: Bayesian Generative Modeling for Neural Density Estimation}

\author{
Chenglin Li \qquad Qiao Liu\thanks{Corresponding author: qiao.liu@yale.edu} \\[2mm]
Department of Biostatistics, Yale University \\
New Haven, Connecticut, USA
}

\begin{document}

\maketitle

\begin{abstract}
Density estimation is a fundamental problem in statistics and machine learning. In this work, we introduce BayesNDE, a neural density estimator based on Bayesian generative modeling. BayesNDE learns a Bayesian generative model and evaluates its density without requiring invertible networks or Jacobian-determinant computation. For each observation, it infers a sample-specific latent posterior to construct an adaptive proposal that focuses computation on regions contributing most to its density. Bridge sampling then combines samples from this proposal with separate posterior samples to estimate the density. Experiments on nonlinear and multimodal synthetic datasets show improved estimation of density values and better recovery of the density structure compared to the state-of-the-art neural density estimators. Applications to real-world datasets further demonstrate improved anomaly detection. Together, these results highlight BayesNDE as a flexible and effective neural density estimator, demonstrating how posterior inference can turn generative models into tools for density estimation. The code and tutorials are available at \url{https://github.com/liuq-lab/BayesNDE}.

\end{abstract}
\section{Introduction}
\label{sec:introduction}

Density estimation is a fundamental problem in statistics and machine learning. Given observations $\{x_i\}_{i=1}^N$ drawn from an unknown distribution, 
the goal is to learn the underlying distribution $p(x)$ and evaluate the density value at a new observation. Traditional density estimators such as kernel-based methods \citep{parzen1962estimation} typically perform well only in low-dimensional settings. Density estimation has become increasingly challenging as modern data exhibit complex patterns, such as nonlinearity and high-dimensionality.

Deep neural networks provide powerful tools for modeling complex distributions, offering new advances to density estimation. There are two major types of neural density estimators: autoregressive models and normalizing flows. Autoregressive models obtain tractable likelihoods through a sequential factorization of the joint density \citep{germain2015made,papamakarios2017masked}, while normalizing flows construct invertible transformations whose Jacobian determinants can be efficiently evaluated \citep{dinh2017density,chen2018neural,chen2019residual}. 

Despite their different formulations, both autoregressive density
estimators and normalizing flows can be viewed within a general
change-of-variables framework~\citep{kingma2016improved}. Specifically, let $z \sim p_Z(z)$ and $x = G(z)$, where 
$G:\mathbb{R}^p \rightarrow \mathbb{R}^p$ is differentiable and invertible. 
The density of $x$ follows,
\begin{equation}
    p_X(x)
    =
    p_Z(z)
    \left|
    \det
    \frac{\partial G(z)}{\partial z^\top}
    \right|^{-1}.
    \label{eq:change_of_variables}
\end{equation}
Within this change-of-variables formulation, exact density evaluation requires both invertible transformation and tractable Jacobian computation. It requires the latent and observed spaces to have the same dimension and constrains the neural architecture, which may potentially limit the expressiveness of the neural architecture.
 
Latent-variable generative models provide an alternative route to
flexible distribution modeling. Given a latent variable
$\vz$ with prior $p_Z(\vz)$ and the model
$p_{\vtheta}(\vx\mid\vz)$, the marginal density is
\begin{equation}
p_{\vtheta}(\vx)
=
\int p_{\vtheta}(\vx\mid\vz)p_Z(\vz)\,d\vz.
\label{eq:marginal_density}
\end{equation}
The challenge is therefore shifted from Jacobian computation to
integration over the latent space, which is generally intractable. Following this direction,
\citet{liu2021density} proposed Roundtrip, which models observations
around a learned low-dimensional manifold and evaluates pointwise
densities using importance sampling or Laplacian approximation.

Although Roundtrip substantially relaxes the architectural constraints
of flow-based density estimators, its importance proposal is centered
at the encoder output, which provides a deterministic latent
representation rather than a distributional characterization of the observation-specific posterior.
The resulting proposal may therefore not fully capture posterior
uncertainty, dependence, or multimodality, potentially reducing
importance-sampling efficiency when proposal-posterior mismatch
is substantial.

To overcome these limitations, we introduce BayesNDE, a neural
density estimator based on Bayesian generative modeling. To our knowledge, BayesNDE is the first neural density-estimation framework to use sample-specific posterior-adaptive bridge sampling
for pointwise density evaluation. We summarize our contributions as follows:

\begin{itemize}
    \item We introduce a new perspective for explicit neural density evaluation by estimating the marginal or conditional density of each observation through posterior-adaptive bridge sampling. This removes the need for dimension-preserving, invertible transformations and tractable Jacobian determinants.

    \item BayesNDE uses observation-specific posterior samples to construct adaptive proposals and bridge sampling to estimate the resulting normalizing constants. We establish consistency of the estimator and characterize how posterior-proposal overlap determines its asymptotic efficiency.

    \item BayesNDE achieves superior performance in a diverse of density estimation benchmarks, including nonlinear and multimodal simulation datasets, and real-world datasets. The downstream applications, such as outlier detection and Bayesian classification, further support the significance of our method.
\end{itemize}

\section{Related Work}
\label{sec:related_work}

\paragraph{Neural density estimation.}

Modern neural density estimators are largely built around model classes that admit tractable likelihood evaluation. Autoregressive models, such as MADE \citep{germain2015made}, represent a joint density through sequential conditional factorizations, while normalizing flows construct invertible transformations with tractable change-of-variables computation, including MAF \citep{papamakarios2017masked}, Real NVP \citep{dinh2017density}, and Residual Flow (Resflow) \citep{chen2019residual}.
Continuous normalizing flows extend this framework to continuous-time transformations \citep{chen2018neural}, while Flow Matching provides an efficient approach for learning the associated vector fields \citep{lipman2023flow}. Relatedly, \citet{wu2017on} used annealed importance sampling to evaluate test log-likelihoods of pretrained decoder-based generative models. Their focus is quantitative evaluation of fitted generative models rather than the development and benchmarking of a general-purpose neural density estimator.

\paragraph{Normalizing-constant estimation.}
Estimating normalizing constants is a classical problem in Bayesian computation. Importance sampling estimates an unknown normalizing constant using weighted samples from a proposal distribution, with efficiency depending critically on proposal-target overlap. Bridge sampling estimates the normalizing constant by combining samples from both a proposal distribution and the normalized target distribution \citep{meng1996simulating,gronau2020bridgesampling}.
Related adaptive importance-sampling methods further use information from posterior regions to construct more effective proposal distributions \citep{hoogerheide2012adaptive}. \citep{jia2020normalizing} adopt bridge sampling with a normalizing-flow-based proposal for normalizing-constant estimation of Bayesian evidence. These methods, however, are designed for generic normalizing-constant estimation rather than neural density estimation of the observed-data
distribution.

\paragraph{Positioning of BayesNDE.}
BayesNDE bridges the gap between neural density estimation and Bayesian normalizing-constant estimation through observation-specific posterior inference. Its generative model is inspired by recent advances in Bayesian generative modeling~\citep{liu2026bayesian,liu2026ai,luo2026bgm,liu2026missingness}, while our focus is on explicit density estimation through posterior-adaptive bridge sampling. BayesNDE exploits this posterior to guide latent-space integration, providing a new perspective on neural density estimation.

\section{Method}
\label{headings}

\subsection{Method overview}
\label{sec:overview}

BayesNDE proceeds in two stages
(Figure~\ref{fig:bayesnde_overview}).
First, we train a Bayesian generative model
to capture the underlying data distribution. Second, with the fitted model, we estimate the density of a new observation as the normalizing constant of its latent posterior. Posterior samples guide the construction of an observation-adaptive proposal that concentrates computation in latent regions contributing most to the marginal-density integral. Bridge sampling then combines proposal draws with held-out posterior
samples to estimate this normalizing constant and obtain the
pointwise density.

\begin{figure}[H]
    \centering
    \includegraphics[width=\linewidth]{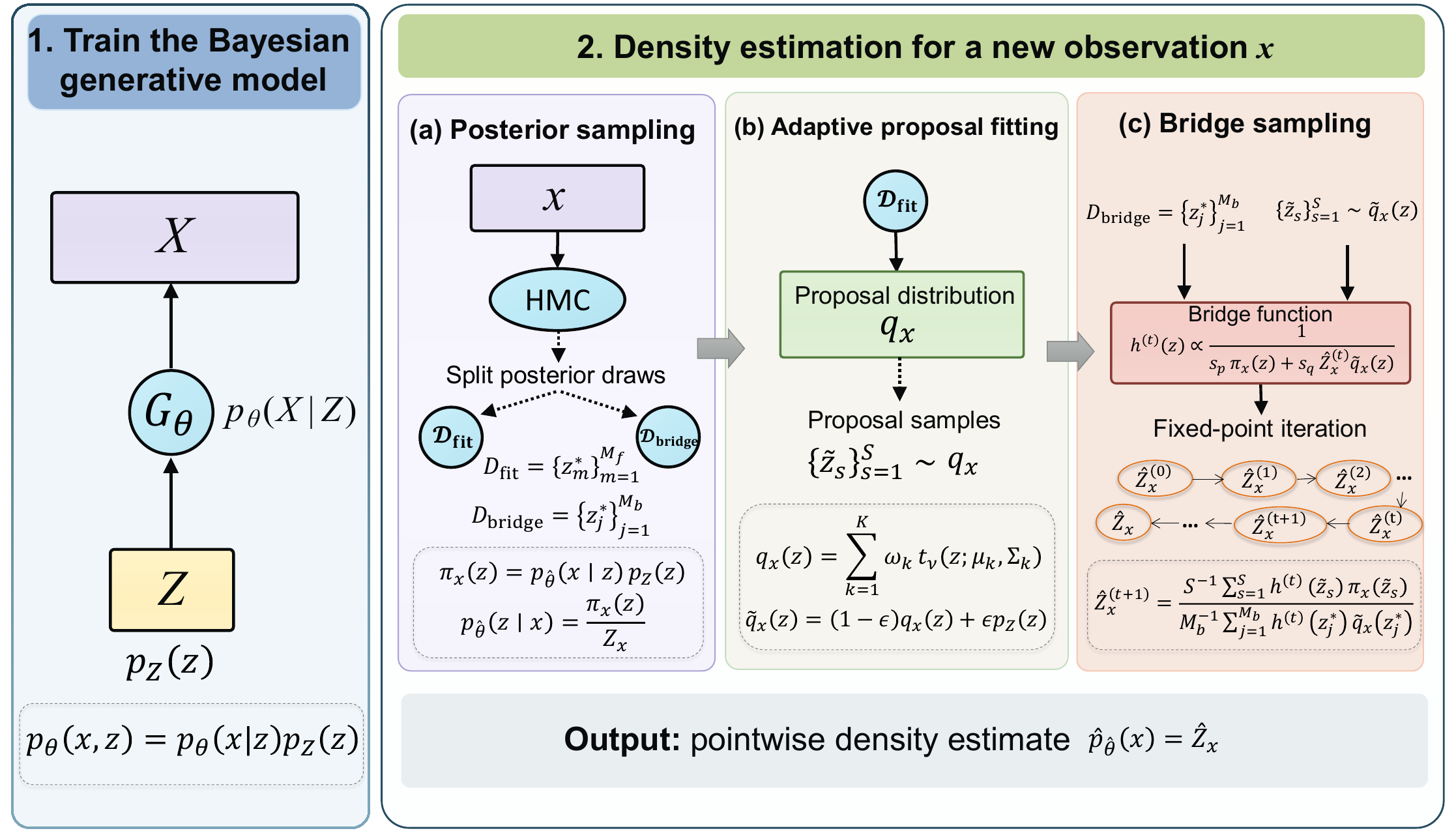}
    \caption{
    Overview of BayesNDE.
    }
    \label{fig:bayesnde_overview}
\end{figure}

Specifically, for a fitted parameter $\hat{\vtheta}$ and an observation $\vx$, define the unnormalized latent posterior as 
    $\pi_{\vx}(\vz)
    =
    p_{\hat{\vtheta}}(\vx\mid\vz)p_Z(\vz)$.
Its normalizing constant is exactly the desired marginal density,
\begin{equation}
    Z_{\vx}
    :=
    \int \pi_{\vx}(\vz)\,d\vz
    =
    p_{\hat{\vtheta}}(\vx),
    \qquad
    p_{\hat{\vtheta}}(\vz\mid\vx)
    =
    \frac{\pi_{\vx}(\vz)}{Z_{\vx}}.
    \label{eq:normalizing_constant}
\end{equation}
Thus, pointwise density evaluation reduces to estimating $Z_{\vx}$. Posterior inference requires only the unnormalized posterior $\pi_{\vx}(\vz)$, allowing posterior samples to identify the latent regions that contribute most to the marginal-density integral. BayesNDE uses these samples to construct an observation-specific proposal
and estimates $Z_{\vx}$ by bridge sampling.

\subsection{Bayesian Generative Model}
\label{sec:bgm}

Let $\rvx \in \R^p$ denote the observed random vector and $\rvz \in \R^d$ a low-dimensional latent random vector, where typically $d < p$. We assume a standard Gaussian prior $p_Z(\vz)=\mathcal{N}(\vz;\mathbf{0},\mI_d)$ for the latent variable. For continuous data, the conditional distribution of $\rvx$ given $\rvz$ is modeled as 
\begin{equation} p_{\vtheta}(\vx \mid \vz) = \mathcal{N} \left( \vx; \vmu_{\vtheta}(\vz), \mSigma_{\vtheta}(\vz) \right), \label{eq:conditional_density} \end{equation}
where $\vtheta$ denotes the parameters of the generative model. Both the conditional mean $\vmu_{\vtheta}(\vz)$ and covariance matrix $\mSigma_{\vtheta}(\vz)$ are represented by a neural network with two output heads. For simplicity, we use the diagonal structure
$\mSigma_{\vtheta}(\vz)
=
\operatorname{diag}
\left(
\sigma_1^2(\vz;\vtheta),\ldots,\sigma_p^2(\vz;\vtheta)
\right)$. Importantly, the Gaussian assumption is imposed conditionally on $\rvz$; a nonlinear marginal distribution of $\vx$ is induced after integrating over the latent variable.

During model fitting, we first use an EGM warm start~\citep{liu2024encoding,liu2026bayesian} to initialize the generative parameters and sample-specific latent representations (see details in Appendix~\ref{app:implementation}).
Then training proceeds by iteratively updating the
sample-specific latent variables $\{\vz_i\}_{i=1}^{N}$ and the shared
generative parameters $\vtheta$. Given the current $\vtheta$, the latent variable associated with observation
$\vx_i$ is updated according to its log-posterior $\log p_{\vtheta}(\vz_i \mid \vx_i)
=
\log p_Z(\vz_i)
+
\log p_{\vtheta}(\vx_i \mid \vz_i)
+
C_i,$
where $C_i$ does not depend on $\vz_i$. Since the right-hand side is differentiable with respect to $\vz_i$,
each latent variable is updated by stochastic gradient ascent. These updates
are fully decoupled across observations and can therefore be performed
in parallel to improve efficiency.

Conditional on the current latent variables, the generative parameters are updated using the mini-batch conditional log-likelihood
$
l_{\mathcal{B}}(\vtheta)
=
\frac{1}{|\mathcal{B}|}
\sum_{i \in \mathcal{B}}
\log p_{\vtheta}(\vx_i \mid \vz_i),
$ together with the regularization terms specified in Appendix~\ref{app:implementation}. The two updates are alternated until convergence. The fitted generative parameters
$\hat{\vtheta}$ are then fixed for subsequent pointwise density evaluation.

\subsection{Posterior Sampling and Proposal Construction}

For a new observation $\vx$, the fitted model defines the latent posterior
\begin{equation}
    p_{\hat{\vtheta}}(\vz\mid\vx)
    =
    \frac{\pi_{\vx}(\vz)}{Z_{\vx}}
    =
    \frac{
        p_{\hat{\vtheta}}(\vx\mid\vz)p_Z(\vz)
    }{
        p_{\hat{\vtheta}}(\vx)
    }.
\label{eq:latent_posterior_test}
\end{equation}
Although the normalizing constant $Z_{\vx}$ is unknown,
posterior sampling only requires evaluation of $\pi_{\vx}(\vz)$.
We therefore use Hamiltonian Monte Carlo (HMC) \citep{duane1987hybrid,neal2011mcmc} to obtain posterior draws
$\{\vz_m^{\star}\}_{m=1}^{M}$ from
$p_{\hat{\vtheta}}(\vz \mid \vx)$. To construct a proposal, the posterior draws are
randomly split into two disjoint sets, $\mathcal{D}_{\mathrm{fit}}$ and
$\mathcal{D}_{\mathrm{bridge}}$, following the
sample-splitting principle studied for bridge sampling
\citep{wong2020properties}. The former is used exclusively for proposal
construction, while the latter is retained for subsequent bridge sampling.

To capture the observation-specific posterior geometry, we first fit a $K$-component full-covariance Gaussian mixture model (GMM)
to $\mathcal{D}_{\mathrm{fit}}$ and use the fitted mixture weights, means,
and regularized covariance matrices to construct a mixture-of-Student-$t$
proposal $
q_{\vx}(\vz)
=
\sum_{k=1}^{K}
\omega_k
t_{\nu}
\left(
\vz;
\vmu_k,
\mSigma_k
\right),$
where $\nu$ denotes the degrees of freedom. The mixture structure accommodates
multimodal posterior geometry, while the heavy-tailed components improve
coverage of the posterior tails.

To improve robustness to proposal misspecification and provide additional
tail coverage, we further introduce a defensive component based on the
latent prior \citep{hesterberg1995weighted,owen2000safe},
$
\widetilde{q}_{\vx}(\vz)
=
(1-\epsilon)q_{\vx}(\vz)
+
\epsilon p_Z(\vz),
\label{eq:defensive_proposal}
$
where $\epsilon \in (0,1)$ controls the defensive weight. We then draw $S$
independent proposal samples
$\{\widetilde{\vz}_s\}_{s=1}^{S}$ from $\widetilde{q}_{\vx}$.
Together with the held-out posterior draws in $\mathcal{D}_{\mathrm{bridge}}$,
these samples are used for pointwise density estimation in the following
section.

\subsection{Bridge-Sampling Density Estimation}

We estimate the normalizing constant
$Z_{\vx}=p_{\hat{\vtheta}}(\vx)$ using bridge sampling. Using the defensive proposal $\widetilde{q}_{\vx}$, we first obtain an importance-sampling estimate
\begin{equation}
\widehat{Z}_{\mathrm{IS}}
=
\frac{1}{S}
\sum_{s=1}^{S}
\frac{
\pi_{\vx}(\widetilde{\vz}_s)
}{
\widetilde{q}_{\vx}(\widetilde{\vz}_s)
},
\qquad
\widetilde{\vz}_s \sim \widetilde{q}_{\vx}.
\label{eq:importance_sampling}
\end{equation}
This estimate is used to initialize the bridge-sampling iteration. For any positive bridge function $h$, the bridge identity and its
Monte Carlo approximation are
\begin{equation}
\begin{aligned}
Z_{\vx}
&=
\frac{
\E_{\widetilde{q}_{\vx}}
\left[
h(\vz)\pi_{\vx}(\vz)
\right]
}{
\E_{p_{\hat{\vtheta}}(\vz \mid \vx)}
\left[
h(\vz)\widetilde{q}_{\vx}(\vz)
\right]
},\qquad
\widehat{Z}_{\vx}
&=
\frac{
S^{-1}\sum_{s=1}^{S}
h(\widetilde{\vz}_s)\pi_{\vx}(\widetilde{\vz}_s)
}{
M_{\mathrm{b}}^{-1}\sum_{j=1}^{M_{\mathrm{b}}}
h(\vz_j^\star)\widetilde{q}_{\vx}(\vz_j^\star)
}.
\end{aligned}
\label{eq:bridge_identity}
\end{equation}

BayesNDE uses the asymptotically optimal bridge function
\citep{meng1996simulating}
\begin{equation}
h(\vz)
\propto
\frac{1}{
s_p\pi_{\vx}(\vz)
+
s_q Z_{\vx}\widetilde{q}_{\vx}(\vz)
},
\qquad
s_p
=
\frac{M_{\mathrm{eff}}}{M_{\mathrm{eff}}+S},
\quad
s_q
=
\frac{S}{M_{\mathrm{eff}}+S},
\label{eq:bridge_function}
\end{equation}
where $M_{\mathrm{eff}}$ denotes an effective posterior sample count computed
from the held-out bridge draws, as detailed in
Appendix~\ref{app:implementation}. Because $h^\star(\vz)$ itself depends on the unknown normalizing constant, $Z_{\vx}$ is obtained through the fixed-point update
\begin{equation}
\widehat{Z}_{\vx}^{(t+1)}
=
\frac{
\displaystyle
\frac{1}{S}
\sum_{s=1}^{S}
\frac{
\pi_{\vx}(\widetilde{\vz}_s)
}{
s_p\pi_{\vx}(\widetilde{\vz}_s)
+
s_q\widehat{Z}_{\vx}^{(t)}
\widetilde{q}_{\vx}(\widetilde{\vz}_s)
}
}{
\displaystyle
\frac{1}{M_{\mathrm{b}}}
\sum_{j=1}^{M_{\mathrm{b}}}
\frac{
\widetilde{q}_{\vx}(\vz_j^\star)
}{
s_p\pi_{\vx}(\vz_j^\star)
+
s_q\widehat{Z}_{\vx}^{(t)}
\widetilde{q}_{\vx}(\vz_j^\star)
}
}.
\label{eq:bridge_update}
\end{equation}
The iteration is initialized at
$\widehat{Z}_{\vx}^{(0)}=\widehat{Z}_{\mathrm{IS}}$, and continued until convergence. The resulting
$\log \widehat{p}_{\hat{\vtheta}}(\vx)=\log\widehat{Z}_{\vx}$ is reported as the estimated pointwise log-density.

For conditional density estimation with an observed conditioning variable
$\vy$, we extend the Bayesian generative model to a conditional generative architecture and the same procedure applies after replacing
$p_{\hat{\vtheta}}(\vx\mid\vz)$ by
$p_{\hat{\vtheta}}(\vx\mid\vz,\vy)$ throughout, yielding an estimate of
$p_{\hat{\vtheta}}(\vx\mid\vy)$. See details in Appendix~\ref{app:conditional_density}.

\subsection{Training Objectives and Model Architecture}
\label{sec:training_architecture}

We briefly summarize the training objectives and network architectures
used to instantiate the Bayesian generative model. For EGM initialization, the generator $G$ and encoder $E$ are trained
using a weighted combination of adversarial matching, bidirectional
reconstruction, distribution matching through MMD
\citep{gretton2012kernel} and sliced-Wasserstein discrepancy
\citep{kolouri2019sliced}, correlation matching, and
conditional-variance regularization:
\begin{equation}
\begin{aligned}
\mathcal{L}_{\mathrm{EGM}}
={}&
\lambda_{\mathrm{adv}}\mathcal{L}_{\mathrm{adv}}
+\lambda_{\mathrm{rec}}\mathcal{L}_{\mathrm{rec}}
+\lambda_{\mathrm{corr}}\mathcal{L}_{\mathrm{corr}}
\\
&+
\lambda_{\mathrm{MMD}}\mathcal{L}_{\mathrm{MMD}}
+\lambda_{\mathrm{SW}}\mathcal{L}_{\mathrm{SW}}
+\lambda_{\mathrm{var}}\mathcal{L}_{\mathrm{logvar}},
\end{aligned}
\label{eq:egm_objective}
\end{equation}
where the adversarial and reconstruction terms encourage agreement
between the data and latent representations, while the remaining terms regularize distributional,
dependence, and conditional-variance
properties. The detailed experimental settings are reported in
Appendix~\ref{app:implementation}.

During the subsequent stochastic iterative updating, the generative parameters are optimized using the conditional negative log-likelihood together with the regularization terms:
\begin{equation}
\mathcal{J}_{\mathrm{iter}}(\vtheta)
=
-l_{\sB}(\vtheta)
+
\lambda_{\mathrm{PP}}\mathcal{L}_{\mathrm{MMD,PP}}
+
\lambda_{\mathrm{corr}}^{\mathrm{iter}}\mathcal{L}_{\mathrm{corr}}
+
\lambda_{\mathrm{var}}^{\mathrm{iter}}\mathcal{L}_{\mathrm{logvar}},
\label{eq:iterative_objective}
\end{equation}
where $\mathcal{L}_{\mathrm{MMD,PP}}$ matches the prior-predictive
distribution to the observed data.

For vector-valued data, the generative model $G_{\vtheta}$ adopts five width-256 hidden layers for
$p\leq10$ and five width-256 residual blocks for $p>10$ with separate conditional-mean and diagonal-variance heads, with a Softplus
transformation to ensure positive variances. Detailed loss definitions,
architecture settings, optimization hyperparameters, and regime-specific
weights are provided in Appendix~\ref{app:implementation}.

\section{Theory}
\subsection{Consistency of Adaptive Bridge Estimation}
\label{sec:theory_consistency}

BayesNDE constructs a bridge proposal from posterior
samples, so the proposal itself is data-adaptive. We show that this adaptation
does not alter the asymptotic target of the density estimator, provided that
the proposal-fitting samples are separated from those used in bridge estimation.

\begin{theorem}[Consistency of posterior-adaptive bridge estimation]
\label{thm:adaptive_bridge_consistency}
Fix an observation $x$ and a fitted generative model $\hat{\vtheta}$.
Let $\pi_x(z)=p_{\hat{\vtheta}}(x\mid z)p_Z(z)$
and
$Z_x=\int \pi_x(z)\,dz=p_{\hat{\vtheta}}(x)$.
Suppose that the adaptive proposal is constructed from a separate posterior
sample, has support wherever $\pi_x$ is positive, and that the proposal and
held-out posterior samples satisfy the corresponding laws of large numbers.
Then the bridge estimate $\widehat Z_x$ produced by BayesNDE satisfies
\[
\widehat Z_x \xrightarrow{p} Z_x
=
p_{\hat{\vtheta}}(x)
\]
as the numbers of proposal and posterior samples tend to infinity.
\end{theorem}

Theorem~\ref{thm:adaptive_bridge_consistency} shows that posterior adaptation
affects the efficiency of bridge estimation, but not its asymptotic target.
The result is conditional on the fitted generative model and therefore concerns
the consistency of the Monte Carlo density evaluation step, rather than
consistency with respect to the unknown data-generating density.
The full assumptions and proof under the regularity conditions are given in Appendix~\ref{app:adaptive_bridge_proof}.

\subsection{Posterior--proposal overlap and estimation efficiency}

Classical bridge-sampling theory shows that the bridge function and the overlap between the two sampling distributions determine the efficiency of normalizing-constant estimation \citep{meng1996simulating,gronau2020bridgesampling}. We specialize this analysis to the latent posterior $p_{\boldsymbol{x}} = \pi_{\boldsymbol{x}}/Z_{\boldsymbol{x}}$ and the defensive proposal $\tilde q_{\boldsymbol{x}}$, treated as fixed given $\mathcal{D}_{\mathrm{fit}}$. Let $N = M_b + S$ with $M_b/N \to s_p \in (0,1)$ and $s_q = 1 - s_p$, and define, for any nonnegative function $q$,
\begin{equation}
  \mathcal{O}_{\boldsymbol{x}}(q) = \int \frac{p_{\boldsymbol{x}}(\boldsymbol{z})\,q(\boldsymbol{z})}{s_p\,p_{\boldsymbol{x}}(\boldsymbol{z}) + s_q\,q(\boldsymbol{z})}\,\mathrm{d}\boldsymbol{z}.
\end{equation}

\begin{theorem}[Efficiency under posterior--proposal overlap]
\label{thm:efficiency}
Suppose the $M_b$ held-out draws are i.i.d.\ from $p_{\boldsymbol{x}}$ and the $S$ proposal draws are i.i.d.\ from $\tilde q_{\boldsymbol{x}}$, independently of each other. For the optimal bridge estimator of \citet{meng1996simulating},
\[
  \sqrt{N}\bigl(\widehat Z_{\boldsymbol{x}}/Z_{\boldsymbol{x}} - 1\bigr) \xrightarrow{d} \mathcal{N}\Bigl(0,\ \tfrac{1}{s_p s_q}\bigl\{\mathcal{O}_{\boldsymbol{x}}(\tilde q_{\boldsymbol{x}})^{-1} - 1\bigr\}\Bigr).
\]
Moreover, $0 < \mathcal{O}_{\boldsymbol{x}}(\epsilon\,p_Z) \le \mathcal{O}_{\boldsymbol{x}}(\tilde q_{\boldsymbol{x}}) \le 1$, where the lower bound does not depend on the fitted mixture $q_{\boldsymbol{x}}$ and the upper bound is attained if and only if $\tilde q_{\boldsymbol{x}} = p_{\boldsymbol{x}}$ almost everywhere.
\end{theorem}

No overlap or moment conditions are required, since the bridge integrands are bounded by $1/s_p$ and $1/s_q$. Stronger overlap between $\tilde q_{\boldsymbol{x}}$ and $p_{\boldsymbol{x}}$ yields smaller Monte Carlo variance, which motivates observation-specific proposals, while the defensive component bounds the variance however poorly $q_{\boldsymbol{x}}$ fits the posterior. For correlated HMC draws, replacing $M_b$ by $M_{\mathrm{eff}}$ in $(s_p, s_q)$ is a heuristic. The proof is given in Appendix~\ref{app:bridge_efficiency_proof}.

\section{Results}
\subsection{Experimental setup}

\paragraph{Datasets and evaluation.}
We evaluate BayesNDE on both synthetic benchmarks with known densities and real-world benchmarks. The synthetic experiments are designed to reflect two complementary challenges in density estimation. The independent Gaussian mixture (Independent GMM) represents a strongly multimodal
distribution: each coordinate independently follows an equally weighted
three-component Gaussian mixture with means $-1$, $0$, and $1$, yielding
$3^p$ modes in $p$ dimensions.
The involute distribution instead represents a highly nonlinear density
concentrated around a curved manifold. Detailed data generative processes are provided in Appendix ~\ref{app: indep_gmm and involute}. For each distribution, we generate $20{,}000$ i.i.d.\ observations and use $81\%/9\%/10\%$ for training, validation, and testing. 

In two dimensions, estimated densities are compared directly with the ground truth on a common grid; for the independent Gaussian mixture, we further evaluate dimensions from $2$ to $25$ using the Spearman rank correlation between true and estimated test densities. For real data, we use five tabular benchmarks from the UCI Machine Learning
Repository \citep{dua2017uci}. BANK and ParkTele are used for unconditional density estimation, while Pendigits10, Vehicle, and EEGEye are used for conditional density estimation with one-hot encoded class labels. Performance is measured by average test log-likelihood.

\paragraph{Baseline methods.}
We compare BayesNDE with six leading neural density estimators spanning autoregressive models, normalizing flows, continuous flows, and generative latent-variable models: MADE~\citep{germain2015made}, MAF~\citep{papamakarios2017masked}, Real NVP~\citep{dinh2017density}, Residual Flow (Resflow) ~\citep{chen2019residual}, Conditional Flow Matching (CFMs)~\citep{lipman2023flow}, and Roundtrip~\citep{liu2021density}. These baselines cover widely used approaches based on autoregressive factorization, invertible transformations, continuous-time flows, and flexible deep generative mappings. Implementation details of the baselines are provided in Appendix~\ref{app:baseline impl}.

\subsection{Simulation studies}

We first evaluate whether the estimated densities recover the underlying
structure of the two synthetic distributions. Figure~\ref{fig:involute} shows density maps of the true and estimated densities in two dimensions. BayesNDE closely recovers both the nonlinear geometry of the involute distribution and the nine separated modes of the Gaussian mixture, while competing methods exhibit varying degrees of smoothing, distortion, or spurious connections between modes. For example, on the involute benchmark, BayesNDE achieves a Spearman correlation of $0.959$, compared with $0.898$ for the strongest baseline, Roundtrip. 

We next examine how density estimation accuracy changes as the dimension increases. Using the Gaussian mixture benchmark, each coordinate independently follows a three-component mixture,
the number of modes grows exponentially as $3^p$, making the problem increasingly challenging with dimension. As shown in Figure~\ref{fig:indep_gmm_highdim}, BayesNDE achieves the highest Spearman correlation at every evaluated dimension. For example, when $p=20$, BayesNDE attains a correlation of $0.57$, compared with $0.43$ for the strongest competing method.

\begin{figure}[H]
    \centering
    \includegraphics[width=\linewidth]{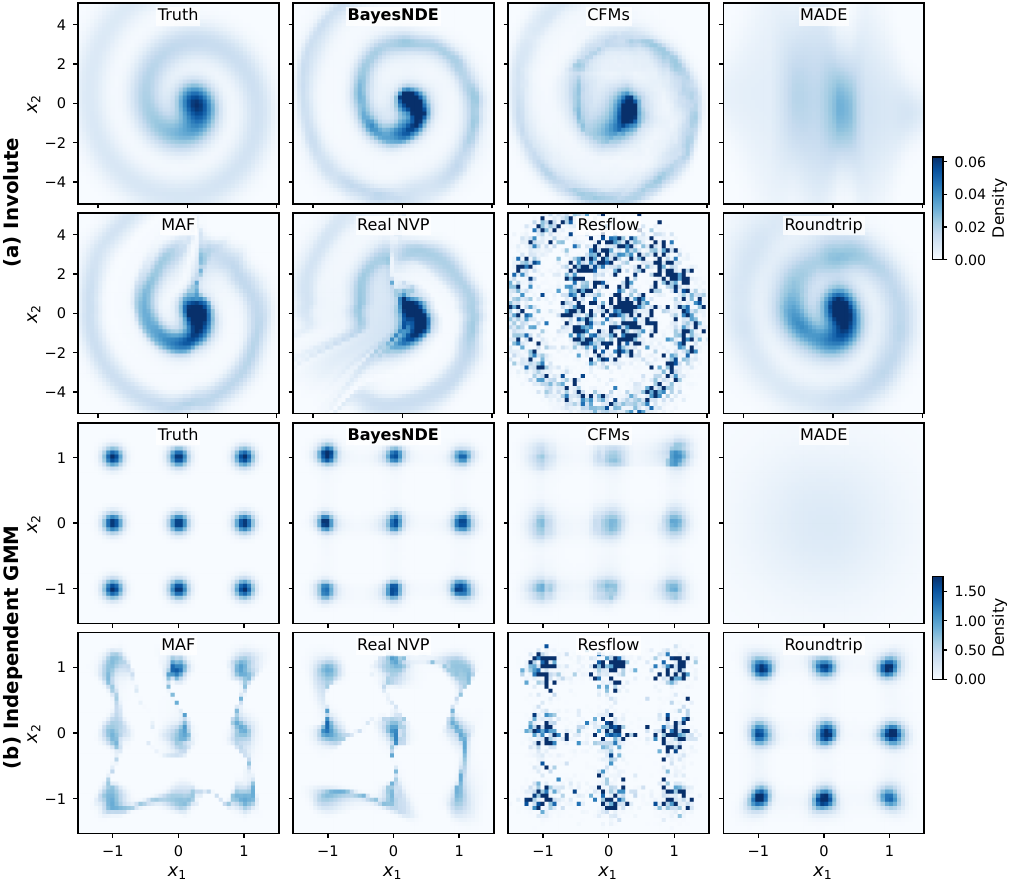}
    \caption{True and estimated densities for the two-dimensional involute distribution and independent Gaussian mixture.}
    \label{fig:involute}
\end{figure}

\subsection{Real datasets}
\label{real datasets}
Table~\ref{tab:real_data} reports average test log-likelihoods on the five real-world datasets, covering both unconditional and conditional density estimation. BayesNDE achieves the highest value on every dataset. The improvements are substantial across tasks. For instance, BayesNDE obtains the highest average test log-likelihood of $41.43$ on BANK, compared with $35.74$ for the best baseline Roundtrip. 

On the three conditional datasets, BayesNDE improves over the strongest competing method by $3.68$-$19.01$ nats. These consistent gains across datasets and task types demonstrate the effectiveness of BayesNDE for real-world density estimation.

\subsection{Density-Based Downstream Tasks}

We first evaluate the practical utility of BayesNDE for outlier detection on three ODDS benchmark datasets \citep{rayana2016odds}: Shuttle, Cardio, and Pendigits. All density estimators were trained only on the training split, and test observations were ranked by estimated log-density, with lower-density observations more likely to be outliers. We compared BayesNDE with five generative density estimators as well as two dedicated anomaly-detection methods, one-class SVM \citep[OC-SVM;][]{scholkopf2001estimating} and Isolation Forest \citep[I-Forest;][]{liu2008isolation}, using precision@$k$, which is the proportion of true anomalies among the $k$ highest-ranked candidates.

\begin{table}[h]
\caption{Average test log-likelihood (nats) on real-world datasets. Uncertainties are two standard
errors of the mean test log-likelihood across test points. The best result for each dataset is shown in
bold.}
\label{tab:real_data}
\begin{center}
\small
\begin{tabular}{lccccc}
\hline
\multicolumn{1}{c}{Method}
& \multicolumn{2}{c}{\bf Unconditional}
& \multicolumn{3}{c}{\bf Conditional}
\\
\cline{2-3}\cline{4-6}
& {\bf BANK} & {\bf ParkTele} & {\bf Pendigits10} & {\bf Vehicle} & {\bf EEGEye}
\\ \hline

MADE
& 14.38$\pm$0.27
& 34.34$\pm$0.67
& -54.56$\pm$0.96
& -51.84$\pm$1.41
& -50.25$\pm$0.18 \\

MAF
& -53.15$\pm$0.20
& 36.99$\pm$0.66
& -52.29$\pm$0.78
& -52.30$\pm$2.25
& -49.47$\pm$0.21 \\

Real NVP
& 27.21$\pm$0.38
& 36.45$\pm$0.68
& -54.59$\pm$1.16
& -63.03$\pm$3.61
& -49.87$\pm$0.22 \\

Resflow
& 23.68$\pm$0.18
& 36.25$\pm$0.63
& -47.77$\pm$0.64
& -45.97$\pm$1.13
& -48.33$\pm$0.29 \\

CFMs
& 10.23$\pm$0.07
& 16.32$\pm$0.19
& -54.40$\pm$0.36
& -49.65$\pm$0.87
& -51.45$\pm$0.14 \\

Roundtrip
& 35.74$\pm$0.13
& 49.64$\pm$0.33
& -48.59$\pm$0.08
& -40.10$\pm$0.70
& -44.83$\pm$0.08 \\

BayesNDE
& \textbf{41.43$\pm$0.17}
& \textbf{59.54$\pm$0.53}
& \textbf{-28.76$\pm$0.44}
& \textbf{-33.54$\pm$0.88}
& \textbf{-41.15$\pm$0.19} \\

\hline
\end{tabular}
\end{center}
\end{table}

\begin{table}[!htbp]
\caption{Precision@k on the ODDS anomaly-detection benchmarks, where k equals the number of outliers in the test split. Results are means over the same three random seeds for every method; OC-SVM is deterministic. Higher values are better.}
\label{tab:anomaly_detection}
\centering
\small
\setlength{\tabcolsep}{2.8pt}
\begin{tabular}{@{}lcccccccc@{}}
\toprule
 & OC-SVM & I-Forest & Real NVP & MAF & CFMs & Resflow & Roundtrip & BayesNDE \\
\midrule
Shuttle
& 0.953 & 0.958 & 0.735 & 0.861 & 0.932 & 0.477 & 0.958 & \textbf{0.965} \\
Cardio
& 0.533 & \textbf{0.556} & 0.178 & 0.200 & 0.511 & 0.244  & 0.311 & \textbf{0.556}  \\
Pendigits
& 0.154 & 0.128 & 0.000 & 0.000 & 0.000 & 0.000 & 0.000 & \textbf{0.180} \\
\bottomrule
\end{tabular}
\vspace{-10pt}
\end{table}

As shown in Table~\ref{tab:anomaly_detection}, BayesNDE achieves the highest precision@$k$ on all three datasets, including 0.556 on Cardio and 0.180 on the challenging Pendigits dataset. These results show that the density estimates produced by BayesNDE are informative for downstream outlier detection and remain competitive with methods designed specifically for this task.
\begin{figure}[H]
    \centering
    \includegraphics[width=\linewidth]{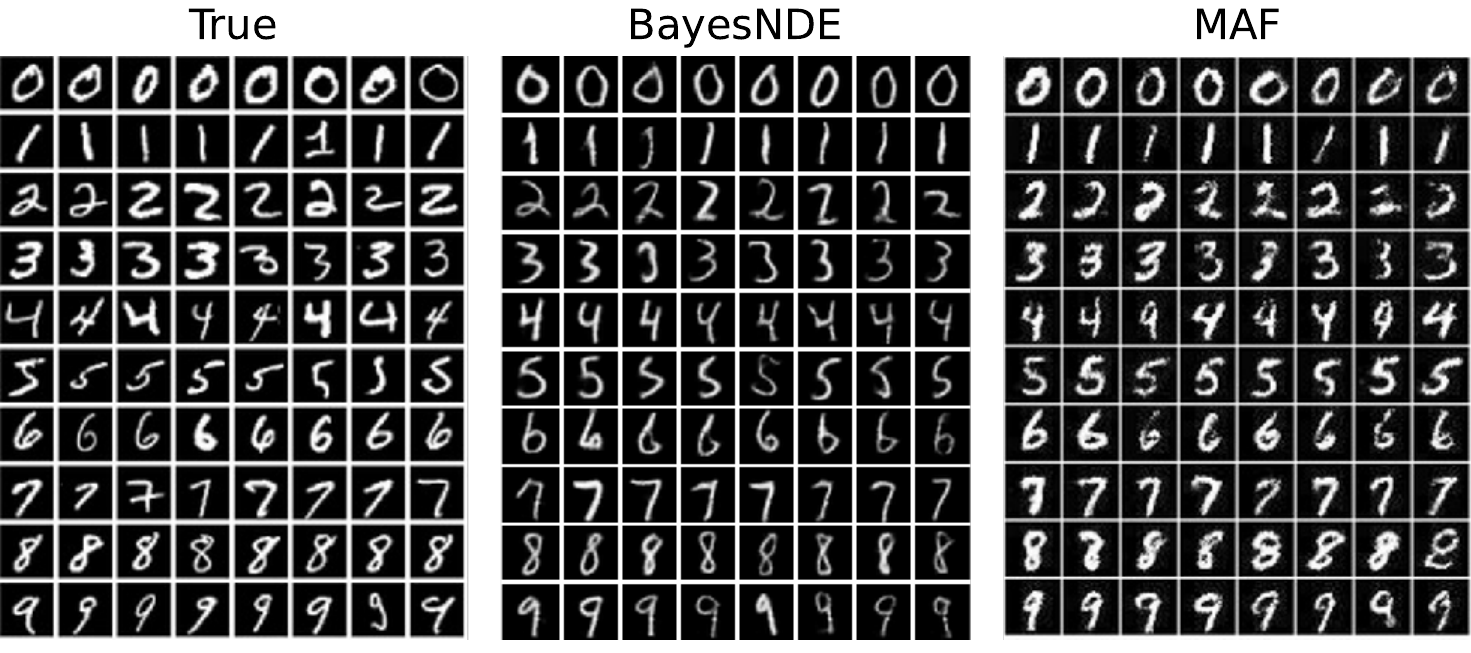}
    \caption{Conditional image generation on MNIST. The first column shows real MNIST examples; the remaining columns show samples generated by BayesNDE and MAF. For each digit class, representative generated images are shown in decreasing order of conditional density $p_\vtheta(\vx \mid y)$.}
    \label{fig:mnist_generation}
\end{figure}
Next, we evaluate BayesNDE on the MNIST image dataset for conditional generation and classification. As shown in Figure~\ref{fig:mnist_generation}, BayesNDE generates diverse and recognizable digits, with higher-density samples exhibiting more representative class-specific structure. With the uniform class prior $p(y)=1/10$, we classify each test image using $\hat y=\arg\max_y p(y\mid\vx)=\arg\max_y\{\log\widehat p_\vtheta(\vx\mid y)+\log p(y)\}=\arg\max_y\log\widehat p_\vtheta(\vx\mid y)$. BayesNDE achieves 98.6\% test accuracy, compared with 98.0\% for Roundtrip and 92.6\% for MAF.

\begin{figure}[H]
    \centering
    \begin{subfigure}[t]{0.50\textwidth}
        \centering
        \includegraphics[width=\linewidth]{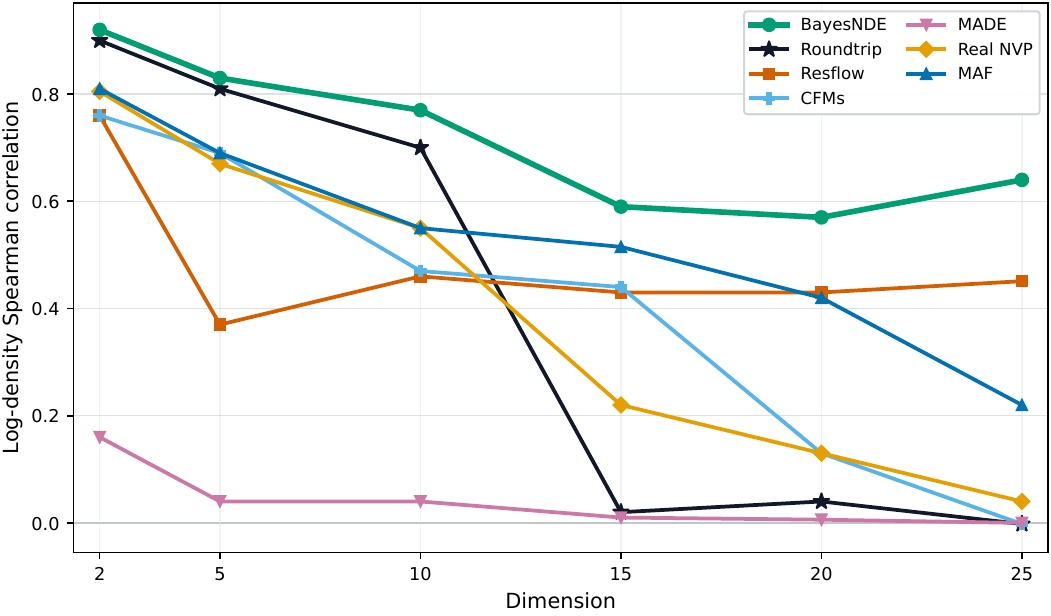}
        \caption{Spearman rank correlation on independent GMM.}
        \label{fig:indep_gmm_highdim}
    \end{subfigure}
    \hfill
    \begin{subfigure}[t]{0.40\textwidth}
        \centering
        \includegraphics[width=\linewidth]{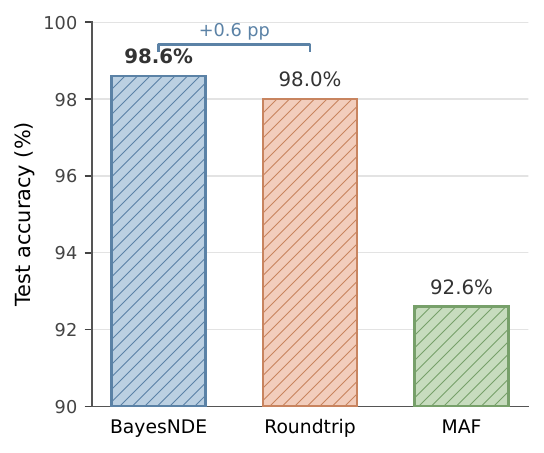} 
        \caption{Test accuracy on MNIST.}
        \label{fig:mnist_acc}
    \end{subfigure}
    
    \caption{Empirical evaluation across benchmarks. (a) Density estimation accuracy across varying dimensions; (b) Classification performance on MNIST.}
    \label{fig:overall_eval}
    \vspace{-15pt}
\end{figure}

\subsection{Ablation Analysis}

We ablate the main components of BayesNDE on the synthetic mixture setting, using the same 500 held-out observations for all comparisons. Table~\ref{tab:ablation} reports the density estimation accuracy. Bridge sampling attains higher ranking accuracy than importance sampling at
every posterior budget, reaching the best accuracy importance sampling attains
at any budget with four times fewer posterior samples (see details in Appendix ~\ref{app:implementation} S6). 
In the 2D ablation setting, removing the EGM warm start severely degrades density recovery. In 15 dimensions, removing correlation matching or the MMD terms also reduces density accuracy.
\begin{table}[htbp]
\centering
\caption{Ablation results on synthetic data.}
\label{tab:ablation}
\small
\setlength{\tabcolsep}{6pt}
\renewcommand{\arraystretch}{1.08}

\begin{tabular}{lc}
\toprule
Method & Spearman $\uparrow$ \\
\midrule
BayesNDE (2D)                & \textbf{0.9214} \\
\quad w.o. EGM init.         & 0.0129 \\
\quad w.o. var. reg.         & 0.8929 \\
\midrule
BayesNDE (15D)               & \textbf{0.6311} \\
\quad w.o. BS                & 0.5932 \\
\quad w.o. EGM reg. (MMD)    & 0.5848 \\
\quad w.o. EGM reg. (Corr)   & 0.6169 \\
\bottomrule
\end{tabular}

\end{table}

\section{Discussion}

We introduced BayesNDE, which casts pointwise density evaluation as posterior normalizing-constant estimation. Observation-specific posterior proposals and bridge sampling enable density evaluation without invertible architectures or Jacobian determinants, while improving density estimation and downstream performance.

Its main limitation is running time: BayesNDE is slower than CFMs, MAF and Roundtrip (Appendix~\ref{app:ISBS}). However, test-time posterior sampling and bridge estimation are independent across observations and can be further accelerated by parallel computing. Accuracy also depends on posterior exploration and proposal overlap in complex latent spaces. Future work will focus on faster training, posterior sampling and reusable proposals to reduce the cost of repeated density evaluation.

\section*{Acknowledgments}

Q.L. was supported in part by the National Human Genome Research
Institute (NHGRI) under Award Number R00HG013661, the Operations Core of the
Claude D. Pepper Older Americans Independence Center at Yale School of Medicine
(P30AG021342), and the YSPH Transformation Pilot Grant.

\bibliographystyle{plainnat}
\bibliography{references}
\appendix
\section{Appendix}

\subsection{Proofs of theory}
\textbf{Proof of Theorem ~\ref{thm:adaptive_bridge_consistency}}
\label{app:adaptive_bridge_proof}
We provide the technical conditions and proof for
Theorem~\ref{thm:adaptive_bridge_consistency}.
Throughout, the observation $x$ and the fitted generative parameter
$\hat{\vtheta}$ are fixed. Define
\[
\pi_x(z)
=
p_{\hat{\vtheta}}(x\mid z)p_Z(z),
\qquad
Z_x
=
\int \pi_x(z)\,dz,
\qquad
p_x(z)
=
\frac{\pi_x(z)}{Z_x}.
\]
By construction, $Z_x=p_{\hat{\vtheta}}(x)$.

Let $\mathcal D_{\mathrm{fit}}$ denote the posterior sample used to construct
the observation-specific proposal. Conditional on $\mathcal D_{\mathrm{fit}}$,
the fitted proposal $q_x$ and the defensive proposal
\[
\widetilde q_x(z)
=
(1-\epsilon)q_x(z)+\epsilon p_Z(z)
\]
are treated as fixed. For notational simplicity, write
$q(z)=\widetilde q_x(z)$.

Conditional on $\mathcal D_{\mathrm{fit}}$, assume:

\begin{enumerate}
    \item[(A1)] $q(z)>0$ whenever $\pi_x(z)>0$.

    \item[(A2)] The proposal draws
    $\widetilde z_1,\ldots,\widetilde z_S$ are i.i.d.\ from $q$, and the
    held-out posterior draws $z_1^\star,\ldots,z_M^\star$ satisfy the
    law of large numbers under $p_x$ for every bounded measurable
    function appearing below.

    \item[(A3)] The bridge weights $s_p$ and $s_q=1-s_p$ remain bounded
    away from zero: for some constant $\eta\in(0,1/2)$,
    \[
        \eta\le s_p,s_q\le 1-\eta.
    \]

    \item[(A4)] The fixed-point iteration is initialized by the
    importance-sampling estimator
    \[
        \widehat Z_{\mathrm{IS}}
        =
        \frac{1}{S}
        \sum_{s=1}^S
        \frac{\pi_x(\widetilde z_s)}{q(\widetilde z_s)},
    \]
    and the iteration is run to a converged fixed point
    $\widehat Z_x$.
\end{enumerate}

Condition (A2) allows the held-out posterior sample to be generated by an
ergodic Markov chain; independence of the posterior draws is not required.
When the proposal-fitting and bridge samples are obtained from a common
Markov chain, (A2) is understood conditionally on
$\mathcal D_{\mathrm{fit}}$.

For $c>0$, define
\[
A_S(c)
=
\frac{1}{S}
\sum_{s=1}^{S}
\frac{\pi_x(\widetilde z_s)}
{s_p\pi_x(\widetilde z_s)+s_qc\,q(\widetilde z_s)}
,
B_M(c)
=
\frac{1}{M}
\sum_{j=1}^{M}
\frac{q(z_j^\star)}
{s_p\pi_x(z_j^\star)+s_qc\,q(z_j^\star)}.
\]
The empirical bridge map is
\[
\widehat T_{M,S}(c)
=
\frac{A_S(c)}{B_M(c)}.
\]

Its population counterparts are
\[
A(c)
=
\mathbb E_q
\left[
\frac{\pi_x(Z)}
{s_p\pi_x(Z)+s_qc\,q(Z)}
\right]
,
B(c)
=
\mathbb E_{p_x}
\left[
\frac{q(Z)}
{s_p\pi_x(Z)+s_qc\,q(Z)}
\right].
\]

We first establish the key population identity. Since
$p_x(z)=\pi_x(z)/Z_x$,
\begin{align}
B(c)
&=
\int
\frac{q(z)}
{s_p\pi_x(z)+s_qc\,q(z)}
\frac{\pi_x(z)}{Z_x}\,dz
\nonumber\\
&=
\frac{1}{Z_x}
\int
\frac{\pi_x(z)q(z)}
{s_p\pi_x(z)+s_qc\,q(z)}
\,dz
\nonumber\\
&=
\frac{A(c)}{Z_x}.
\label{eq:appendix_population_relation}
\end{align}
Consequently,
\begin{equation}
T(c)
:=
\frac{A(c)}{B(c)}
=
Z_x
\qquad
\text{for every } c>0.
\label{eq:appendix_population_constant}
\end{equation}
Thus the population bridge map is constant: the choice of bridge iterate
$c$ and the fitted proposal affect Monte Carlo efficiency but not the
population target.

We next show that the empirical bridge map converges uniformly to this
population map on compact subsets of $(0,\infty)$.
Let $K=[a,b]\subset(0,\infty)$ and define
\[
f_c(z)
=
\frac{\pi_x(z)}
{s_p\pi_x(z)+s_qc\,q(z)},
\qquad
g_c(z)
=
\frac{q(z)}
{s_p\pi_x(z)+s_qc\,q(z)}.
\]
For all $c\in K$,
\[
0\le f_c(z)\le \frac{1}{s_p}\le\frac{1}{\eta}, 
0\le g_c(z)\le\frac{1}{s_qc}
\le\frac{1}{\eta a}.
\]
Hence both families are uniformly bounded.

Moreover, they are uniformly Lipschitz in $c$. Indeed,
\[
\left|
\frac{\partial f_c(z)}{\partial c}
\right|
=
\frac{s_q\pi_x(z)q(z)}
{\left[s_p\pi_x(z)+s_qc\,q(z)\right]^2}.
\]
Using $4uv\le(u+v)^2$ with
$u=s_p\pi_x(z)$ and $v=s_qc\,q(z)$ gives
\[
\left|
\frac{\partial f_c(z)}{\partial c}
\right|
\le
\frac{1}{4s_pa}
\le
\frac{1}{4\eta a}.
\]
Similarly,
\[
\left|
\frac{\partial g_c(z)}{\partial c}
\right|
=
\frac{s_q q(z)^2}
{\left[s_p\pi_x(z)+s_qc\,q(z)\right]^2}
\le
\frac{1}{s_q a^2}
\le
\frac{1}{\eta a^2}.
\]
Therefore the classes $\{f_c:c\in K\}$ and
$\{g_c:c\in K\}$ are uniformly bounded and equicontinuous in their
one-dimensional parameter $c$.

By the law of large numbers in (A2), convergence holds for every fixed
$c$. A finite-grid argument together with the preceding uniform
Lipschitz bounds then yields
\begin{equation}
\sup_{c\in K}
|A_S(c)-A(c)|
\xrightarrow{p}0
\label{eq:appendix_uniform_A}
\end{equation}
and
\begin{equation}
\sup_{c\in K}
|B_M(c)-B(c)|
\xrightarrow{p}0.
\label{eq:appendix_uniform_B}
\end{equation}

It remains to control the denominator. By (A1), the integrand defining
$A(c)$ is strictly positive on the support of $\pi_x$. Hence
$A(c)>0$ for every $c>0$. Since $A(c)$ is continuous in $c$, compactness
of $K$ gives $\inf_{c\in K}A(c)>0.
$
From \eqref{eq:appendix_population_relation},
\[
\inf_{c\in K}B(c)
=
\frac{1}{Z_x}
\inf_{c\in K}A(c)
>0.
\]
Combining this fact with
\eqref{eq:appendix_uniform_A}--\eqref{eq:appendix_uniform_B} and the
continuous mapping theorem gives
\begin{equation}
\sup_{c\in K}
\left|
\widehat T_{M,S}(c)-Z_x
\right|
\xrightarrow{p}0.
\label{eq:appendix_uniform_bridge}
\end{equation}

We now connect the uniform convergence result to the actual BayesNDE
iteration. First, the importance-sampling initializer is consistent.
Indeed, by (A1),
\[
\mathbb E_q
\left[
\frac{\pi_x(Z)}{q(Z)}
\right]
=
\int \pi_x(z)\,dz
=
Z_x,
\]
and the random variable $\pi_x(Z)/q(Z)$ is integrable because its
expectation equals the finite constant $Z_x$. Therefore, by the law of
large numbers,
\begin{equation}
\widehat Z_{\mathrm{IS}}
\xrightarrow{p} Z_x.
\label{eq:appendix_is_consistency}
\end{equation}

Choose, for example,
\[
K=
\left[
\frac{Z_x}{2},
\frac{3Z_x}{2}
\right].
\]
By \eqref{eq:appendix_is_consistency},
$\widehat Z_{\mathrm{IS}}\in K$ with probability tending to one.
Furthermore, by \eqref{eq:appendix_uniform_bridge},
\[
\sup_{c\in K}
\left|
\widehat T_{M,S}(c)-Z_x
\right|
<
\frac{Z_x}{4}
\]
with probability tending to one. On this event,
\[
\widehat T_{M,S}(K)
\subset
\left[
\frac{3Z_x}{4},
\frac{5Z_x}{4}
\right]
\subset K.
\]
Thus, once initialized in $K$, every subsequent bridge iterate remains
in $K$. If the iteration converges to the reported fixed point
$\widehat Z_x$, then
\[
\widehat Z_x
=
\widehat T_{M,S}(\widehat Z_x),
\qquad
\widehat Z_x\in K,
\]
and hence
\[
|\widehat Z_x-Z_x|
\le
\sup_{c\in K}
\left|
\widehat T_{M,S}(c)-Z_x
\right|
\xrightarrow{p}0.
\]
Therefore,
\[
\widehat Z_x
\xrightarrow{p}
Z_x
=
p_{\hat{\vtheta}}(x)
.
\]

\hfill$\square$

\textbf{Proof of Theorem~\ref{thm:efficiency}}
\label{app:bridge_efficiency_proof}

Fix $x$ and abbreviate
\[
p(z)=p_{\hat{\vtheta}}(z\mid x),
\qquad
q(z)=\widetilde q_x(z),
\qquad
Z=Z_x,
\]
so that $\pi_x(z)=Zp(z)$. Here
\[
\widetilde q_x(z)
=
(1-\epsilon)q_x(z)+\epsilon p_Z(z)
\]
is the defensive proposal used by BayesNDE.
Let $N=M_b+S$, with $M_b/N\to s_p\in(0,1)$ and
$S/N\to s_q=1-s_p$.

For independent posterior and proposal samples, the asymptotically optimal
bridge function of \citet{meng1996simulating} is proportional to
\[
h^\star(\vz)
=
\frac{1}{s_p\pi_x(z)+s_qZq(z)}
\propto
\frac{1}{s_p p(z)+s_q q(z)}.
\]
Define
\[
d(z)=s_p p(z)+s_q q(z),
\qquad
\mathcal O
=
\int \frac{p(z)q(z)}{d(z)}\,dz.
\]
Under $h^\star(\vz)$, the relative bridge estimator can be written as
\begin{equation}
\frac{\widehat Z}{Z}
=
\frac{\frac1S\sum_{s=1}^S U(\widetilde z_s)}
     {\frac1{M_b}\sum_{j=1}^{M_b} V(z_j^\star)},
\qquad
U(z)=\frac{p(z)}{d(z)},\quad
V(z)=\frac{q(z)}{d(z)}.
\label{eq:app_bridge_ratio}
\end{equation}
Since
$\mathbb E_q[U]=\mathbb E_p[V]=\mathcal O$, the joint central limit
theorem gives
\[
\sqrt N
\begin{pmatrix}
\overline U_S-\mathcal O\\
\overline V_{M_b}-\mathcal O
\end{pmatrix}
\xrightarrow{d}
\mathcal N\!\left(
0,
\begin{pmatrix}
\operatorname{Var}_q(U)/s_q & 0\\
0 & \operatorname{Var}_p(V)/s_p
\end{pmatrix}
\right).
\]
Applying the delta method to $g(a,b)=a/b$ at
$(\mathcal O,\mathcal O)$ yields
\begin{equation}
\sqrt N\left(\frac{\widehat Z}{Z}-1\right)
\xrightarrow{d}
\mathcal N(0,\sigma^2),
\qquad
\sigma^2
=
\frac{1}{\mathcal O^2}
\left[
\frac{\operatorname{Var}_q(U)}{s_q}
+
\frac{\operatorname{Var}_p(V)}{s_p}
\right].
\label{eq:app_bridge_delta}
\end{equation}

The variance simplifies using
\begin{align}
s_p\mathbb E_q[U^2]+s_q\mathbb E_p[V^2]
&=
\int
\frac{s_p p(z)^2q(z)+s_qp(z)q(z)^2}{d(z)^2}\,dz \nonumber\\
&=
\int\frac{p(z)q(z)}{d(z)}\,dz
=
\mathcal O.
\label{eq:app_overlap_identity}
\end{align}
Together with
$\mathbb E_q[U]=\mathbb E_p[V]=\mathcal O$ and
$s_p+s_q=1$, this gives
\[
\frac{\operatorname{Var}_q(U)}{s_q}
+
\frac{\operatorname{Var}_p(V)}{s_p}
=
\frac{\mathcal O-\mathcal O^2}{s_p s_q}.
\]
Substitution into \eqref{eq:app_bridge_delta} therefore yields
\[
\sqrt{M_b+S}
\left(
\frac{\widehat Z_x}{Z_x}-1
\right)
\xrightarrow{d}
\mathcal N\!\left(
0,\,
\frac{1}{s_p s_q}
\left[
\mathcal O_x(\widetilde q_x)^{-1}-1
\right]
\right).
\]

It remains to characterize the overlap. First, for fixed $p(z)\ge 0$,
define
\[
f_z(q)
=
\frac{p(z)q}
{s_p p(z)+s_q q},
\qquad q\ge 0.
\]
For $p(z)>0$,
\[
\frac{\partial f_z(q)}{\partial q}
=
\frac{s_p p(z)^2}
{\left[s_p p(z)+s_q q\right]^2}
\ge 0,
\]
so the overlap integrand is nondecreasing in its second argument.
Because the defensive proposal satisfies
\[
\widetilde q_x(z)
=
(1-\epsilon)q_x(z)+\epsilon p_Z(z)
\ge
\epsilon p_Z(z)
\]
pointwise, it follows that
\[
\mathcal O_x(\widetilde q_x)
\ge
\mathcal O_x(\epsilon p_Z).
\]
Moreover, under the Gaussian latent prior used by BayesNDE,
$p_Z(z)>0$ for all $z$, and hence
$\mathcal O_x(\epsilon p_Z)>0$.

For the upper bound, by the weighted harmonic--arithmetic mean
inequality,
\[
\frac{p(z)q(z)}
{s_p p(z)+s_q q(z)}
=
\left(
\frac{s_p}{q(z)}
+
\frac{s_q}{p(z)}
\right)^{-1}
\le
s_p q(z)+s_q p(z),
\]
with the left-hand side defined as zero whenever $p(z)q(z)=0$.
Integrating gives
\[
\mathcal O_x(q)\le s_p+s_q=1.
\]
Equality holds if and only if $p(z)=q(z)$ almost everywhere.
Taking $q=\widetilde q_x$ therefore gives
\[
0<
\mathcal O_x(\epsilon p_Z)
\le
\mathcal O_x(\widetilde q_x)
\le1,
\]
with the upper bound attained if and only if
$\widetilde q_x=p_x$ almost everywhere.

\subsection{Datasets in experiments}

\textbf{S1. Data Generating Processes for Synthetic Distributions}
\label{app: indep_gmm and involute}
\paragraph{Independent Gaussian mixture.}
Following \citep{liu2021density}, we consider an independent Gaussian mixture (Independent GMM) distribution and extend the original two-dimensional construction to $p$ dimensions. Specifically, we generate $\mathbf{X}=(X_1,\ldots,X_p)^\top$ with mutually independent coordinates, where each coordinate follows the same three-component Gaussian mixture, $X_j \sim \frac{1}{3}\mathcal{N}(-1,0.1^2)+\frac{1}{3}\mathcal{N}(0,0.1^2)+\frac{1}{3}\mathcal{N}(1,0.1^2)$, for $j=1,\ldots,p$. Equivalently, the joint density is $$p(\mathbf{x})=\prod_{j=1}^{p}\left[\frac{1}{3}\sum_{\mu\in\{-1,0,1\}}\phi(x_j;\mu,0.1^2)\right]$$, where $\phi(\cdot;\mu,\sigma^2)$ denotes the Gaussian density with mean $\mu$ and variance $\sigma^2$. This construction yields $3^p$ equally weighted modes located at the Cartesian product $\{-1,0,1\}^p$. As $p$ increases, the number of modes therefore grows exponentially, providing an increasingly challenging benchmark for evaluating the ability of density estimators to recover highly multimodal distributions in higher dimensions.

\paragraph{Involute distribution.}
The involute distribution is constructed by first sampling a latent variable $R\sim\mathrm{Uniform}(0,2\pi)$ and then independently generating $X_1\mid R=r\sim\mathcal{N}(r\sin(2r),0.4^2)$ and $X_2\mid R=r\sim\mathcal{N}(r\cos(2r),0.4^2)$. Thus, conditional on $R=r$, the observation is centered along the nonlinear curve $(r\sin(2r),r\cos(2r))$ with isotropic Gaussian perturbation. The corresponding marginal density is $$p(\mathbf{x})=\frac{1}{2\pi}\int_0^{2\pi}\phi(x_1;r\sin(2r),0.4^2)\phi(x_2;r\cos(2r),0.4^2)\,dr$$, which can be evaluated numerically when the ground-truth density is required. In contrast to the isolated modes of Independent GMM, this distribution concentrates probability mass around a highly nonlinear curved structure and is therefore used to assess the ability of density estimators to capture complex nonlinear geometries.

\textbf{S2. Details of data preprocessing for the outlier-detection
datasets used in our study.}

\medskip
\noindent
\textit{Shuttle.}
The Shuttle dataset
(\url{http://odds.cs.stonybrook.edu/shuttle-dataset/})
is derived from the Statlog Shuttle classification dataset and contains
nine numerical features. The original training and test sets were combined
before constructing the outlier-detection benchmark. Samples from the five
smallest classes, namely classes 2, 3, 5, 6, and 7, were grouped together
as outliers, while samples from class 1 were treated as inliers. Samples
from class 4 were discarded. The resulting dataset contains 49,097
observations, including 3,511 outliers (7.15\%).

\medskip
\noindent
\textit{Cardio.}
The Cardio dataset
(\url{http://odds.cs.stonybrook.edu/cardio-dataset/})
is derived from the UCI Cardiotocography dataset and consists of 21
numerical features extracted from fetal-heart-rate and uterine-contraction
measurements. The original observations were classified by expert
obstetricians into normal, suspect, and pathologic fetal states. For the
ODDS outlier-detection benchmark, the suspect class was discarded, the
1,655 normal observations were treated as inliers, and the 176 pathologic
observations were treated as outliers. The resulting dataset contains
1,831 observations with an outlier proportion of 9.61\%.

\medskip
\noindent
\textit{Pendigits.}
The Pendigits dataset
(\url{http://odds.cs.stonybrook.edu/pendigits-dataset/})
is derived from the UCI Pen-Based Recognition of Handwritten Digits
dataset. Each observation is represented by 16 numerical features
corresponding to eight sampled two-dimensional coordinates along a
handwritten pen trajectory. The original dataset contains ten digit
classes, from 0 to 9. In the ODDS benchmark, digit 0 was designated as
the outlier class and downsampled to 156 observations, whereas the 6,714
observations corresponding to digits 1--9 were treated as inliers. The
resulting dataset contains 6,870 observations with an outlier proportion
of 2.27\%.

\medskip
\noindent
\textit{Common preprocessing and data splitting.}
For each dataset, every feature was transformed to the interval $[0,1]$
using feature-wise minimum and maximum values. The observations and their labels were then jointly permuted
using NumPy random seed 0. After permutation, the final 10\% of the
observations were retained as the test set. Of the remaining observations,
the final 10\% were used for validation and the rest were used for
training. This produces an approximately 81/9/10
training/validation/test split. The resulting split sizes are reported
in Table~\ref{tab:odds-data}. Class labels were not
provided as inputs to any of the density-estimation models.

\begin{table}[htbp]
\centering
\caption{Basic descriptions of the three datasets used for outlier detection in our study.}
\label{tab:odds-data}
\begin{tabular}{lccccccc}
\toprule
 & & & & & \multicolumn{3}{c}{\# examples} \\
\cmidrule(lr){6-8}
Dataset & Dim($z$) & Dim($x$) & Outliers (\%) & Test outliers ($k$) & Train & Validation & Test \\
\midrule
Shuttle   & 3 & 9  & 7.15 & 365 & 39770 & 4418 & 4909 \\
Cardio    & 5 & 21 & 9.61 & 15  & 1484  & 164  & 183  \\
Pendigits & 5 & 16 & 2.27 & 13  & 5565  & 618  & 687  \\
\bottomrule
\end{tabular}

\vspace{2pt}
{\footnotesize Outliers (\%) is the outlier rate of the full dataset; Test outliers ($k$)
is the number of outliers in the test split, which sets $k$ in precision@$k$.
Dim($z$) denotes the latent dimension used by Roundtrip. For BayesNDE, we set
Dim($z$) = Dim($x$), corresponding to 9, 21, and 16 for Shuttle, Cardio, and
Pendigits, respectively.}
\end{table}

\textbf{S3. Details of data preprocessing for the Real word
datasets used in our study.}

\medskip
\noindent\textit{BANK.} The BANK dataset relates to a marketing campaign of a
Portuguese banking institution, where the goal is to predict whether the client
will subscribe a deposit. Label encoding was used for discrete features in the
raw data with values between 0 and \texttt{n\_classes}. Then a uniform noise of
$(-0.2, 0.2)$ was added to each feature. At last, the raw data was applied a
feature scaling through a min--max normalization and randomly split into 90\%
training set and 10\% test. Note that for neural density estimators, 10\% of the
training set is kept for validation. BANK is used for unconditional density
estimation.

\medskip
\noindent\textit{ParkTele.} The Parkinsons Telemonitoring dataset (UCI archive
189) contains biomedical voice measurements from 42 patients with early-stage
Parkinson's disease, recruited for a six-month trial of a telemonitoring device.
Each example comprises 16 dysphonia measures derived from a sustained phonation:
five jitter variants, six shimmer variants, the noise-to-harmonics and
harmonics-to-noise ratios, and the nonlinear measures RPDE, DFA and PPE. Two of
the 16 were removed because they are definitional multiples of other columns:
\texttt{Jitter:DDP} equals $3\times$\texttt{Jitter:RAP} and
\texttt{Shimmer:DDA} equals $3\times$\texttt{Shimmer:APQ3} (the observed ratios
lie in $[2.9697, 3.0294]$ and $[2.9942, 3.0062]$, departing from exactly three
only through the five-decimal rounding of the published file). Retaining them
makes the design matrix numerically singular, so that a continuous density on
$\mathbb{R}^{16}$ is not well defined. The remaining 14 features were scaled by a
per-feature min--max normalization, and the same split as the BANK dataset was
used. The subject identifier, age, sex, test time and the two UPDRS scores were
discarded. ParkTele is used for unconditional density estimation. We note that
the 5{,}875 recordings come from 42 subjects, so the random split places
recordings of the same subject in both the training and the test set; all methods
share the identical split.

\medskip
\noindent\textit{Pendigits10.} The Pen-Based Recognition of Handwritten Digits
dataset (UCI archive 81) contains pen trajectories collected from 44 writers on a
pressure-sensitive tablet. Each digit is resampled to eight equally spaced points
along the trajectory, giving 16 integer coordinates in $[0, 100]$, and is
labelled with one of the ten digit classes. Because a continuous density is not
defined on an integer lattice, a uniform noise of $(0, 1)$ was added to each
coordinate with a fixed random seed, so that all methods see identical inputs;
with unit lattice spacing, the expected log-density under uniform dequantization provides a lower bound on the discrete log-probability mass by Jensen's inequality.

\medskip
\noindent\textit{Vehicle.} The Statlog Vehicle Silhouettes dataset (UCI archive
149) describes four vehicle types (bus, Opel, Saab, van) by 18 integer-valued
shape descriptors extracted from silhouette images, including compactness,
circularity, radius ratio, the axis-aligned and principal-axis moments, scatter
ratio, elongatedness, hollows ratio and the skewness and kurtosis about the major
and minor axes. A uniform noise of $(0, 1)$ was added to each feature for the
same reason as in Pendigits10. 

\medskip
\noindent\textit{EEGEye.} The EEG Eye State dataset (UCI archive 264) records one
continuous 117-second measurement from an Emotiv EEG headset, giving 14 electrode
channels per time point, with the eye state (open or closed) determined from a
simultaneous video recording. Four of the 14{,}980 examples are documented
measurement artifacts in which individual channels reach values above
$7\times10^{5}$ against a median near $4.3\times10^{3}$; these were removed using
the criterion that the robust $z$-score $|x - \mathrm{median}| / \mathrm{MAD}$
exceeds 100 in any channel, leaving 14{,}976 examples. 

\medskip
\noindent\textit{MNIST.} MNIST (http://yann.lecun.com/exdb/mnist/) contains
$28\times28$ grayscale images of handwritten digits. The images were flattened
into 784-dimensional vectors and scaled to $[0, 1]$.

\medskip
\noindent The descriptions of the five tabular datasets and the image dataset
(MNIST), including feature dimension and sample size, are summarized in
Table~\ref{tab:dataset_description}.

\begin{table}[H]
\caption{Basic descriptions of the 6 datasets for density estimation used in our study.
BANK and ParkTele are used for unconditional density estimation; 
Pendigits10, Vehicle and EEGEye for conditional density estimation, where
$\dim(\vx)$ counts only the continuous features and the class label is the
conditioning variable.}
\label{tab:dataset_description}
\begin{center}

\begin{tabular}[H]{llccrrr}
\hline
\multirow{2}{*}{Dataset} & \multirow{2}{*}{Domain} & \multirow{2}{*}{Dim($\vz$)} & \multirow{2}{*}{Dim($\vx$)} & \multicolumn{3}{c}{\# of examples} \\
\cline{5-7}
 & & & & Train & Validation & Test \\
\hline
Vehicle      & Transportation & 6   & 18   & 684    & 77     & 85    \\
ParkTele     & Medicine       & 8   & 14   & 4{,}760  & 528    & 587   \\
Pendigits10  & Handwriting    & 8   & 16   & 8{,}902  & 990    & 1{,}100 \\
EEGEye       & Neuroscience   & 5   & 14   & 12{,}130 & 1{,}348  & 1{,}498 \\
BANK         & Finance        & 8   & 17   & 36{,}621 & 4{,}069  & 4{,}521 \\
MNIST        & Image          & 10 & 784  & 50{,}000 & 10{,}000 & 10{,}000 \\
\hline
\end{tabular}
\end{center}
\end{table}

\begin{table*}[t]
\centering
\caption{Network architecture used by BayesNDE for conditional MNIST image
generation and density estimation. Images are scaled to $[0,1]^{784}$, the
class label $y\in\{0,1\}^{10}$ is one-hot encoded, and the latent variable is
$z\in\mathbb{R}^{10}$. For convolutional layers, $k$, $s$, and $c$ denote
kernel size, stride, and number of output channels, respectively. All
convolutions use same padding. BN denotes batch normalization and LReLU has
negative slope $0.2$. The decoder has two output heads that parameterize a
Gaussian distribution in logit space.}
\label{tab:mnist-bayesnde-architecture}
\renewcommand{\arraystretch}{1.14}
\setlength{\tabcolsep}{5pt}
\small
\begin{tabularx}{\textwidth}{@{}>{\raggedright\arraybackslash}X >{\raggedright\arraybackslash}X@{}}
\toprule
\multicolumn{1}{c}{\textbf{Conditional decoder $G_{\vtheta}$}}
& \multicolumn{1}{c}{\textbf{Image discriminator $D_x$}} \\
\midrule
Inputs: $z\in\mathbb{R}^{10}$ and $y\in\{0,1\}^{10}$
& Inputs: $x\in[0,1]^{784}$ and $y\in\{0,1\}^{10}$ \\
$\operatorname{Concat}(z,y)\in\mathbb{R}^{20}$
& Reshape $x$ to $28\times28\times1$; tile $y$ to $28\times28\times10$ \\
FC $6272$, LReLU
& $\operatorname{Concat}(x,\operatorname{tile}(y))$: $28\times28\times11$ \\
Reshape to $7\times7\times128$
& Conv ($k=5$, $s=2$, $c=32$), LReLU: $14\times14\times32$ \\
Tile $y$ to $7\times7\times10$ and concatenate: $7\times7\times138$
& Conv ($k=5$, $s=2$, $c=64$), LReLU: $7\times7\times64$ \\
Transposed conv ($k=3$, $s=2$, $c=64$), BN, LReLU: $14\times14\times64$
& Conv ($k=3$, $s=2$, $c=128$), LReLU: $4\times4\times128$ \\
Transposed conv ($k=3$, $s=2$, $c=32$), BN, LReLU: $28\times28\times32$
& Flatten to $2048$; concatenate $y$: $2058$ \\
Conv ($k=3$, $s=1$, $c=32$), BN, LReLU
& FC $128$, LReLU \\
Two parallel Conv heads ($k=1$, $s=1$, $c=1$)
& FC $1$ (linear score) \\
Flatten heads to $m_{\vtheta}(z,y),r_{\vtheta}(z,y)\in\mathbb{R}^{784}$
& \\
$v_{\vtheta}=v_{\min}+(v_{\max}-v_{\min})[1-\exp\{-\operatorname{softplus}(r_{\vtheta})\}]$
& \\
$v_{\min}=10^{-4}$ and $v_{\max}=4$
& \\
\midrule
\multicolumn{1}{c}{\textbf{Conditional encoder $E_{\phi}$}}
& \multicolumn{1}{c}{\textbf{Latent discriminator $D_z$}} \\
\midrule
Inputs: $x\in[0,1]^{784}$ and $y\in\{0,1\}^{10}$
& Input: $z\in\mathbb{R}^{10}$ \\
Reshape $x$ to $28\times28\times1$; tile $y$ to $28\times28\times10$
& FC $256$, BN, Tanh \\
$\operatorname{Concat}(x,\operatorname{tile}(y))$: $28\times28\times11$
& FC $256$, BN, Tanh \\
Conv ($k=3$, $s=2$, $c=32$), BN, LReLU: $14\times14\times32$
& FC $128$, BN, Tanh \\
Conv ($k=3$, $s=2$, $c=64$), BN, LReLU: $7\times7\times64$
& FC $64$, BN, Tanh \\
Conv ($k=3$, $s=1$, $c=128$), BN, LReLU: $7\times7\times128$
& FC $1$ (linear score) \\
Flatten to $6272$; concatenate $y$: $6282$
& \\
FC $256$, LReLU
& \\
FC $10$ (linear latent code)
& \\
\midrule
\multicolumn{2}{@{}p{0.97\textwidth}@{}}{
For image generation, a logit
$L\sim\mathcal{N}(m_{\vtheta},\operatorname{diag}(v_{\vtheta}))$ is drawn and
the displayed image is $\operatorname{sigmoid}(L)$.} \\
\bottomrule
\end{tabularx}
\end{table*}

\subsection{Conditional Density Estimation}
\label{app:conditional_density}

BayesNDE extends directly to conditional density estimation when an observed
conditioning variable $y$ is available. We use a conditional generative model
\[
p_{\vtheta}(x\mid z,y),
\qquad z\sim p_Z(z),
\]
where $y$ is provided as an additional input to the generative network.
For a fitted model $\hat{\vtheta}$, the conditional density of interest is
\[
p_{\hat{\vtheta}}(x\mid y)
=
\int
p_{\hat{\vtheta}}(x\mid z,y)p_Z(z)\,dz.
\]
Thus, defining
\[
\pi_{x,y}(z)
=
p_{\hat{\vtheta}}(x\mid z,y)p_Z(z),
\]
its normalizing constant is
\[
Z_{x\mid y}
=
\int \pi_{x,y}(z)\,dz
=
p_{\hat{\vtheta}}(x\mid y),
\]
and the corresponding latent posterior is
\[
p_{\hat{\vtheta}}(z\mid x,y)
=
\frac{\pi_{x,y}(z)}{Z_{x\mid y}}.
\]

Conditional density evaluation therefore follows exactly the same procedure
as in the unconditional case. HMC targets the unnormalized log posterior
$\log p_{\hat{\vtheta}}(x\mid z,y)+\log p_Z(z)$; posterior samples are split
into proposal-fitting and bridge subsets; an observation-specific proposal
$q_{x,y}(z)$ is fitted from the former; and bridge sampling estimates
$Z_{x\mid y}$. No modification of the bridge-sampling estimator is required
beyond replacing $\pi_x$ by $\pi_{x,y}$ throughout. For categorical conditioning variables used in our experiments, $y$ is
one-hot encoded and supplied to the conditional generative network.

\subsection{implementation}
\label{app:implementation}

\textbf{S1. Details of network regimes for the independent gussian
datasets used in our simulation study.}

We used two network regimes determined by the observed dimension $p$.  For
low-dimensional data ($p\leq 10$), the generator $G$ and encoder $E$ were
fully connected multilayer perceptrons with five hidden layers of width 256.
For higher-dimensional data ($p>10$), $G$ and $E$ were residual networks with
five width-256 hidden blocks.  The data-space and latent-space discriminators
used hidden widths $256$--$256$--$128$--$64$ in both regimes.  Hidden layers
used LeakyReLU activations.

\textbf{S2. EGM initialization}

We initialize the generator $G$, encoder $E$, and discriminators $D_x$ and
$D_z$ using encoding generative modeling (EGM). The EGM objective is
\begin{equation}
\begin{aligned}
\mathcal L_{\mathrm{EGM}}
={}&
\lambda_{\mathrm{adv},x}\mathcal L_{\mathrm{adv},x}
+\lambda_{\mathrm{adv},z}\mathcal L_{\mathrm{adv},z}
+\lambda_{\mathrm{rec},x}\mathcal L_{\mathrm{rec},x}
+\lambda_{\mathrm{rec},z}\mathcal L_{\mathrm{rec},z}
\\
&+
\lambda_{\mathrm{corr}}\mathcal L_{\mathrm{corr}}
+\lambda_{\mathrm{var}}\mathcal L_{\mathrm{logvar}}
+\lambda_{\mathrm{MMD,marg}}\mathcal L_{\mathrm{MMD,marg}}
\\
&+
\lambda_{\mathrm{MMD,joint}}\mathcal L_{\mathrm{MMD,joint}}
+\lambda_{\mathrm{SW}}\mathcal L_{\mathrm{SW}},
\end{aligned}
\label{eq:appendix_egm_general}
\end{equation}

\paragraph{Adversarial losses.}
The discriminators minimize the least-squares objective
\begin{equation}
\begin{aligned}
\mathcal L_{D}
={}&\tfrac12\Big\{\mathbb E_{\vx\sim P^*}\big[(0.9-D_x(\vx))^2\big]
      +\mathbb E_{\vz\sim P_Z}\big[(0.1-D_x(G(\vz)))^2\big]\Big\}\\
&+\tfrac12\Big\{\mathbb E_{\vz\sim P_Z}\big[(0.9-D_z(\vz))^2\big]
      +\mathbb E_{\vx\sim P^*}\big[(0.1-D_z(E(\vx)))^2\big]\Big\},
\end{aligned}
\end{equation}
while $G$ and $E$ minimize
\begin{equation}
\mathcal L_{\mathrm{adv},x}=\mathbb E_{\vz\sim P_Z}\big[(0.9-D_x(G(\vz)))^2\big],
\qquad
\mathcal L_{\mathrm{adv},z}=\mathbb E_{\vx\sim P^*}\big[(0.9-D_z(E(\vx)))^2\big].
\end{equation}
Discriminator and generator/encoder updates alternate one-to-one.

\paragraph{Reconstruction losses.}
\begin{equation}
\mathcal L_{\mathrm{rec},x}
=\mathbb E_{\vx\sim P^*}\Big[\tfrac1p\big\|\vx-G(E(\vx))\big\|_2^2\Big],
\qquad
\mathcal L_{\mathrm{rec},z}
=\mathbb E_{\vz\sim P_Z}\Big[\tfrac1p\big\|\vz-E(G(\vz))\big\|_2^2\Big].
\end{equation}

\paragraph{Correlation matching.}
The correlation matching is
\begin{equation}
\mathcal L_{\mathrm{corr}}
=
\frac{1}{p^2}
\sum_{i=1}^{p}
\sum_{j=1}^{p}
\left[
\operatorname{Corr}(\vx_i',\vx_j')
-
\operatorname{Corr}(\vx_i,\vx_j)
\right]^2,
\qquad
\vx'=G(E(\vx)).
\label{eq:appendix_corr}
\end{equation}
\paragraph{Marginal MMD.}
With the one-dimensional Gaussian kernel $k_h(u,v)=\exp\{-(u-v)^2/(2h^2)\}$ and
bandwidths $\mathcal H=\{0.05,0.1,0.2,0.5,1.0\}$,
\begin{equation}
\mathcal L_{\mathrm{MMD,marg}}
=\frac1{|\mathcal H|}\sum_{h\in\mathcal H}\frac1p\sum_{j=1}^p
\frac1{n^2}\sum_{a,b=1}^n
\Big[k_h(x_{aj},x_{bj})+k_h(\tilde x_{aj},\tilde x_{bj})-2k_h(x_{aj},\tilde x_{bj})\Big].
\end{equation}

\paragraph{Joint MMD.}
Both samples are first standardized coordinatewise by the pooled minibatch mean
$\boldsymbol m$ and variance $\boldsymbol v$ (treated as constants),
$\bar\vx=(\vx-\boldsymbol m)/\sqrt{\boldsymbol v+10^{-6}}$, and likewise $\bar{\tilde\vx}$.
With $\kappa_s(\vu,\vv)=\exp\{-\|\vu-\vv\|_2^2/(2sp)\}$ and $s\in\mathcal S=\{0.5,1,2\}$,
\begin{equation}
\mathcal L_{\mathrm{MMD,joint}}
=\max\Big\{0,\;\frac1{|\mathcal S|}\sum_{s\in\mathcal S}
\frac1{n^2}\sum_{a,b=1}^n
\big[\kappa_s(\bar\vx_a,\bar\vx_b)+\kappa_s(\bar{\tilde\vx}_a,\bar{\tilde\vx}_b)
-2\kappa_s(\bar\vx_a,\bar{\tilde\vx}_b)\big]\Big\}.
\end{equation}

\paragraph{Sliced Wasserstein.}
Using the same standardization, let
$\Theta=[\ve_1,\dots,\ve_p,\boldsymbol\theta_1,\dots,\boldsymbol\theta_{L-p}]$
with $L=64$, where the $\ve_j$ are coordinate axes and the $\boldsymbol\theta_l$ are
fixed random unit vectors. For each direction $\boldsymbol\omega_l$, let
$u^{(l)}_{(1)}\le\dots\le u^{(l)}_{(n)}$ and $v^{(l)}_{(1)}\le\dots\le v^{(l)}_{(n)}$ be the
sorted projections $\boldsymbol\omega_l^\top\bar\vx_a$ and
$\boldsymbol\omega_l^\top\bar{\tilde\vx}_a$. Then
\begin{equation}
\mathcal L_{\mathrm{SW}}
=\frac1{nL}\sum_{l=1}^{L}\sum_{i=1}^{n}\big(u^{(l)}_{(i)}-v^{(l)}_{(i)}\big)^2.
\end{equation}

The conditional-variance stabilization term is
\begin{equation}
\mathcal L_{\mathrm{logvar}}
=
\mathbb E_{\vz\sim P_Z}
\left[
\frac{1}{p}
\sum_{j=1}^{p}
\left\{
\log\!\left(\sigma_j^2(\vz)+10^{-8}\right)
-\log(0.01)
\right\}^{2}
\right].
\label{eq:appendix_logvar}
\end{equation}

Throughout, the generator is stochastic,
$G(\vz)=\boldsymbol\mu(\vz)+\boldsymbol\sigma(\vz)\odot\boldsymbol\epsilon$,
$\boldsymbol\epsilon\sim N(\mathbf 0,I_p)$, with
$\sigma_j^2(\vz)=\operatorname{softplus}(\cdot)+10^{-6}$, and $P_Z=N(\mathbf 0,I_p)$.
All expectations are estimated on minibatches of size $n=256$; we write
$\{\vx_a\}_{a=1}^n\sim P^*$ and $\{\tilde\vx_a=G(\vz_a)\}_{a=1}^n$, $\vz_a\sim P_Z$.
For the high-dimensional independent-GMM experiments,
$\lambda_{\mathrm{adv},x}=\lambda_{\mathrm{adv},z}=1$,
$\lambda_{\mathrm{rec},x}=3$, $\lambda_{\mathrm{rec},z}=1$, $\lambda_{\mathrm{corr}}=0.3$,
$\lambda_{\mathrm{MMD,marg}}=1000$, $\lambda_{\mathrm{MMD,joint}}=100$,
$\lambda_{\mathrm{SW}}=300$, and $\lambda_{\mathrm{var}}=0$. For the low-dimensional independent-GMM experiments, $G$ and $E$ are
five-layer width-256 multilayer perceptrons and the EGM learning rate is
$10^{-3}$. The final configurations use
$\lambda_{\mathrm{adv},x}
=\lambda_{\mathrm{adv},z}
=\lambda_{\mathrm{rec},z}=1$,
$\lambda_{\mathrm{corr}}=0.3$, and
$\lambda_{\mathrm{var}}=0.01$.
The data-space reconstruction weight is
$\lambda_{\mathrm{rec},x}=1$ for the 2D settings and
$\lambda_{\mathrm{rec},x}=3$ for the 5D and 10D settings.
The marginal-MMD penalty is disabled except in the 5D configuration,
where $\lambda_{\mathrm{MMD,marg}}=100$.

For the five real datasets, training uses Adam at learning
rate $10^{-3}$ with batch size $256$, and the shared weights are
$\lambda_{\mathrm{adv},x}=\lambda_{\mathrm{adv},z}=1$,
$\lambda_{\mathrm{rec},x}=3$, $\lambda_{\mathrm{rec},z}=1$, an overall cycle
weight of $5$, a KL weight of $5\times10^{-5}$, and one discriminator update per
generator update. The variance target is $0.01$ with numerical stabilization
$10^{-8}$, and the joint-MMD and sliced-Wasserstein penalties are disabled.
The remaining regularization weights are fixed per dataset to the values in
Table~\ref{tab:real-weights}. The EGM step is chosen by the pooled generation rank mentioned in S4, on validation samples only.

\begin{table}[htbp]
  \centering
  \caption{Regularization weights for the real datasets, fixed on validation
  data before any test point is scored.}
  \label{tab:real-weights}
  \begin{tabular}{l c ccc c ccc}
    \toprule
    & & \multicolumn{3}{c}{EGM} & & \multicolumn{3}{c}{Iterative} \\
    \cmidrule(lr){3-5}\cmidrule(lr){7-9}
    Dataset & $z$ & $\alpha$ & $\lambda_{\mathrm{var}}$ & $\lambda_{\mathrm{corr}}$
    & & $\lambda^{\mathrm{iter}}_{\mathrm{var}}$ & $\lambda^{\mathrm{iter}}_{\mathrm{corr}}$ & $\lambda_{\mathrm{PP}}$ \\
    \midrule
    BANK        & 8 & 0.1 & 0.01 & 0.3 & & 0    & 0.3 & 1000 \\
    ParkTele    &  8 & 0   & 0    & 0.3 & & 0    & 0.3 & 1000 \\
    Pendigits10 &  8 & 0   & 0    & 0.3 & & 0.01 & 0.3 & 0    \\
    Vehicle     &  6 & 0   & 0    & 0   & & 0    & 0   & 0    \\
    EEGEye      &  5 & 0.1 & 0.01 & 0   & & 0.01 & 0   & 0    \\
    \bottomrule
  \end{tabular}
\end{table}

\textbf{S3. Alternating latent and generator updates}

After EGM initialization, the sample-specific latent variables and
generative parameters are refined by alternating stochastic updates.
For a mini-batch $\mathcal B$, the latent representations are updated by
minimizing the negative log posterior
\begin{equation}
\mathcal J_z
=
\frac{1}{|\mathcal B|}
\sum_{i\in\mathcal B}
\left[
-\log p_{\vtheta}(\vx_i\mid\vz_i)
+\frac{1}{2}\|\vz_i\|_2^2
\right].
\label{eq:appendix_latent_update}
\end{equation}
The latent variables are optimized using Adam with learning rate $0.005$.

Conditional on the updated latent variables, the generator is optimized using
\begin{equation}
\mathcal J_{iter}
=
-\frac{1}{|\mathcal B|}
\sum_{i\in\mathcal B}
\log p_{\vtheta}(\vx_i\mid\vz_i)
+
\lambda_{\mathrm{PP}}\mathcal L_{\mathrm{MMD,PP}}
+
\lambda_{\mathrm{corr}}^{\mathrm{iter}}
\mathcal L_{\mathrm{corr}}
+
\lambda_{\mathrm{var}}^{\mathrm{iter}}
\mathcal L_{\mathrm{logvar}}.
\label{eq:appendix_generator_update}
\end{equation}
Here,
$\mathcal L_{\mathrm{MMD,PP}}$
is a prior-predictive multi-scale MMD penalty comparing the observed
mini-batch with samples generated from fresh
$\vz\sim\mathcal N(0,I)$.
Terms not enabled in a given configuration have zero weight. For the involute distribution, the generator is optimized using AdamW with a constant learning rate of $10^{-6}$.
For the independent Gaussian mixture, it is optimized using Adam with learning rate $0.005$, kept constant for $p\ge 15$ and decayed as $0.005\,(1+t/10)^{-0.6}$ at epoch $t$ for $p\le 10$.

For the real datasets, the iterative regularization weights are likewise fixed
per dataset (Table~\ref{tab:real-weights}). Every run starts from the EGM
checkpoint frozen in the preceding stage, uses Adam at learning rate $0.005$
with batch size $256$, and is trained for $2{,}000$ epochs with a checkpoint
every $50$. The reported epoch is the one maximizing the validation decoder log-likelihood, frozen before any test point is scored.

\textbf{S4. Step/Epoch selections}

All architecture, regularization, EGM-step, and iterative-epoch decisions were
made without access to the test observations.  Selection used fixed validation
observations and fixed random seeds.

For simulation experiments, both the EGM checkpoint and the iterative-updating checkpoint were selected with
the same criterion, computed from held-out observations and generated samples.
For a checkpoint grid $\mathcal{E}$ and the diagnostic set $\mathcal{M}$, we minimized the
equal-weight average percentile rank
\begin{equation}
 \widehat e
 =\arg\min_{e\in\mathcal E}
 \frac{1}{|\mathcal M|}\sum_{m\in\mathcal M}
 r_m(e),
 \qquad
 r_m(e)=\frac{\bigl|\{e'\in\mathcal E: d_m(e')<d_m(e)\}\bigr|}{|\mathcal E|-1},
\end{equation}
where $d_m(e)$ is diagnostic $m$ at checkpoint $e$, oriented so that smaller is better.
The diagnostics are RBF MMD, mean symmetric KL, sliced Wasserstein distance, mean
marginal Wasserstein distance, mean KS statistic, dimension-normalized correlation
error, one minus the local inverse Simpson's index \citep[LISI;][]{korsunsky2019fast}, computed on the union of held-out and generated samples labeled by their source and rescaled to $[0,1]$ so that $1$ indicates perfect mixing, and one minus the fraction of adequately reproduced
marginal standard deviations.  This criterion requires neither ground-truth density
nor known mixture centres.  For the conditional datasets every diagnostic was computed
per class and averaged with weights proportional to the observed class frequencies. The test split was evaluated only after all choices had been frozen. For Real UCI datasets,  The EGM configuration and step are chosen by the validation generation-diagnostic rank sum; the iterative configuration and epoch by the validation decoder log-likelihood.

\textbf{S5. Posterior simulation and bridge estimation}

For each evaluation point, HMC was initialized near $E(\vx)$ using four chains,
step size $0.003$, and 10 leapfrog steps. The independent-GMM configuration retained 1600 posterior draws after 800
burn-in transitions.  Half of the retained posterior draws were used to fit the
proposal and the other half were reserved for bridge estimation. The involute and UCI evaluations used four HMC chains, an initial step size of $0.003$, 10 leapfrog steps, 800 burn-in transitions per chain, and 1600 retained draws in total, divided equally between proposal fitting and bridge estimation. For MNIST, we used four chains with the same initial step size and number of leapfrog steps, 400 burn-in transitions per chain, 800 retained draws in total (400 for proposal fitting and 400 for bridge estimation), and 5000 proposal draws. In all experiments, chains were initialized from the encoder with independent Gaussian perturbations of standard deviation $0.01$, and the step size was adapted during the first 80\% of burn-in toward a target acceptance probability of $0.75$.

The observation-specific proposal was constructed by fitting a
five-component, full-covariance Gaussian mixture to the proposal-fitting HMC
draws.  Its fitted weights and means were retained, and its Gaussian kernels
were replaced by Student-$t$ kernels with three degrees of freedom.  Let
$\widehat\mSigma_k^{\mathrm{GMM}}$ denote the covariance returned by the GMM
fit.  The GMM was fitted with diagonal covariance regularizer
$\delta=10^{-3}$.  Before Cholesky factorization, the implementation further
formed $ \mL_k\mL_k^\top
 =\widehat\mSigma_k^{\mathrm{GMM}}+\eta_k\mI,
 \eta_k\geq\delta$, starting with $\eta_k=\delta$ and increasing it by factors of ten only if the
factorization failed.  Thus the scale used by the proposal includes both the
GMM covariance regularization and the additional numerical jitter; it is not
the unregularized within-component sample covariance.

The proposal was augmented by a standard-normal defensive
component of weight $0.05$, and 20,000 proposal samples were drawn.  The bridge
fixed point was initialized by the importance-sampling estimate and iterated on
the log scale until numerical convergence, with tolerance $10^{-5}$ and at most
1000 iterations.

The bridge fractions were formed using a quantity \citep{vehtari2021rank} denoted by
$M_{\mathrm{eff}}$.  In the implementation, draws from all HMC chains
are pooled, randomly permuted, and divided into proposal-fitting and bridge
subsets.  The bridge subset is treated as a single sequence, an effective
sample size is computed separately for each latent coordinate, and the minimum
over coordinates is used as $M_{\mathrm{eff}}$, truncated to the interval
$[1,M_b]$.  The resulting fractions are
$s_p=\frac{M_{\mathrm{eff}}}{M_{\mathrm{eff}}+S}, s_q=\frac{S}{M_{\mathrm{eff}}+S}.$

Because the permutation does not preserve within-chain temporal order, this
quantity should be understood as a heuristic weighting parameter computed from the pooled and shuffled bridge subset, rather than a conventional multi-chain Markov-chain ESS.
All posterior kernels, proposal densities, importance weights, and bridge
updates were evaluated on the log scale.

\textbf{S6. Bridge Sampling Ablation and Running Time}
\label{app:ISBS}

In table \ref{tab:ablation}, Both estimators are evaluated on the same $500$ held-out
test points with $M$ = 400 and the same defensive mixture proposal, so the training recipe is held
fixed and the only difference is how the normalising constant
$Z = p_\vtheta(x_{\mathrm{obs}})$ is estimated from the posterior draws. For picture ~\ref{fig:isbs-budget} (a), only the HMC
budget $M$ varies; the number of proposal draws is held at the production value
$S = 20000$ throughout. Each point is the mean over three independent Monte Carlo
replicates of the estimator, with the model, data, test points and budgets held
identical across replicates. Beyond $M = 400$ the rank correlation of bridge sampling
saturates, and the remaining differences are within about one standard error across
replicates, as the estimate converges to the fitted model's density $p_{\hat{\theta}}$
rather than to the true density.

\begin{figure}[htbp]
  \centering
  \includegraphics[width=\linewidth]{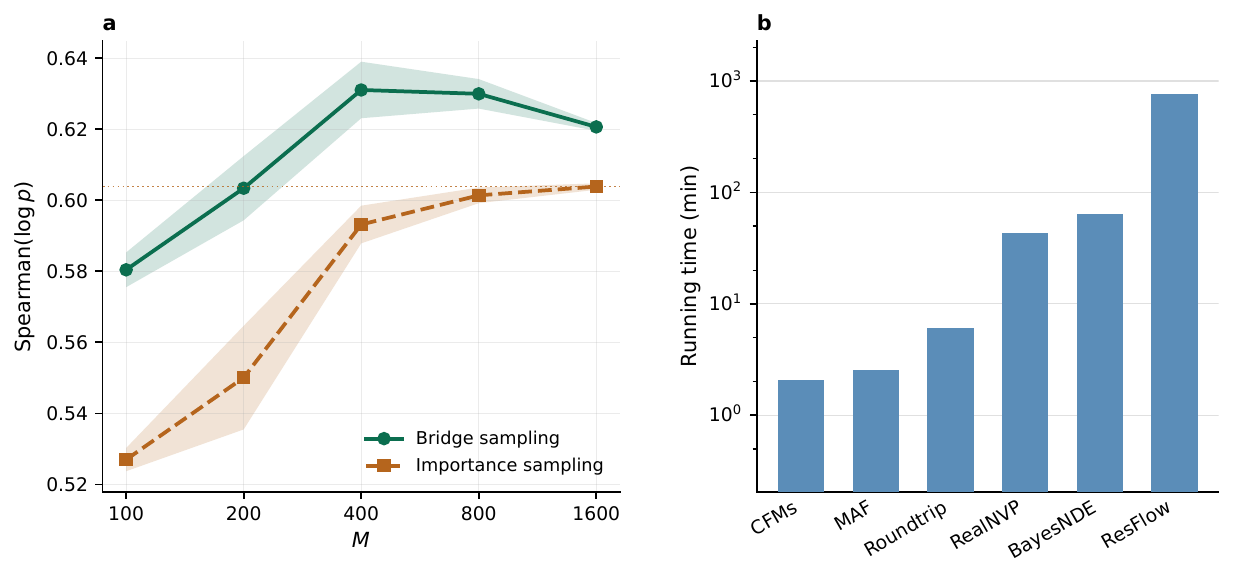}
  \caption{\textbf{(a)} Sampling-budget curve on the $15$-dimensional independent GMM
  benchmark: rank correlation between $\log \hat p$ and the ground-truth log-density
  against the posterior budget $M$. Shaded bands are $\pm 1$ standard error over three
  estimator replicates, and the dotted horizontal line marks the highest accuracy
  importance sampling attains at any budget, which bridge sampling already exceeds at
  $M = 400$. \textbf{(b)} Running time on BANK: the time to fit each estimator plus
  the time the fitted estimator takes to score $500$ test points, drawn uniformly at random without
  replacement from the $4521$-row test split. All six methods use all $17$ columns and
  were timed on one node of the same type (Intel Xeon 6442Y, four cores, a single NVIDIA
  A40). The vertical axis is logarithmic.}
  \label{fig:isbs-budget}
\end{figure}

\begin{table}[htbp]
  \centering
  \caption{Sampling-budget sweep at $p = 15$, mean $\pm$ standard error over
  three estimator replicates on $500$ test points. $\Delta$ is bridge minus
  importance sampling.}
  \label{tab:isbs-budget}
  \begin{tabular}{r ccc}
    \toprule
    & \multicolumn{2}{c}{Spearman($\log p$)} & \\
    \cmidrule(lr){2-3}
    $M$ & Bridge & IS & $\Delta$ \\
    \midrule
    $100$  & $\mathbf{0.5804} \pm 0.0049$ & $0.5270 \pm 0.0033$ & $+0.0535$ \\
    $200$  & $\mathbf{0.6034} \pm 0.0090$ & $0.5501 \pm 0.0146$ & $+0.0534$ \\
    $400$  & $\mathbf{0.6311} \pm 0.0080$ & $0.5932 \pm 0.0053$ & $+0.0379$ \\
    $800$  & $\mathbf{0.6300} \pm 0.0042$ & $0.6014 \pm 0.0022$ & $+0.0286$ \\
    $1600$ & $\mathbf{0.6207} \pm 0.0011$ & $0.6039 \pm 0.0009$ & $+0.0168$ \\
    \bottomrule
  \end{tabular}
\end{table}

Panel~(b) of Figure~\ref{fig:isbs-budget} turns to cost, reporting for each method the
time to fit the model plus the time to score the $500$ test points. The axis is
logarithmic because the totals span two and a half orders of magnitude, from $2.1$
minutes for CFMs to $756$ minutes for Resflow. Testing is timed with the whole $500 \times 17$ test array handed to each method at once,
on one card. The flows score it in a single forward pass and CFMs in a single solve of its
density ODE. Roundtrip draws its $500 \times S$ importance samples as one array, and
BayesNDE advances the $500 \times 4$ HMC chains in lockstep, returning a
$500 \times M \times 8$ array of posterior draws, then carries its proposal draws and
bridge iterations as $500 \times S$ arrays.

The two Monte Carlo estimators each have a sampling budget to choose, and one rule picks
it for both: the cheapest budget whose per-point estimates still rank the test set in
agreement with that method's own converged estimate at Spearman $\ge 0.85$. Roundtrip's rank
agreement with its own converged estimate is $0.700$ at $2\,500$ samples and $0.758$ at
$10\,000$, so it is timed at $40\,000$; BayesNDE reaches $0.850$ at $M = 100$ and is
timed there. At either budget testing is a small part of the total --- $8.6$ seconds for
Roundtrip and $20.0$ seconds for BayesNDE, against $6.0$ and $63.8$ minutes of training. The training part covers model fitting only. BayesNDE runs in $64.2$ minutes in total. That is longer than CFMs, MAF and Roundtrip
($2.1$, $2.6$ and $6.1$ minutes) and about one and a half times RealNVP's $43.6$ minutes, and
twelve times shorter than Resflow's $756$ minutes. On the same data, BayesNDE's mean test log-likelihood exceeds Roundtrip's by $5.7$ nats.

\subsection{Baseline methods}
\label{app:baseline impl}
We compare BayesNDE against six density-estimation baselines spanning
autoregressive flows, coupling flows, free-form residual flows, continuous-time
flow matching, and the GAN-based estimator that is closest to our own setting.
Every method is trained on exactly the same train/validation split, is evaluated
on a byte-identical set of held-out points (we assert equality of the SHA-256
digests of the evaluation matrices across all methods). All baselines run their original authors' code, taken unmodified from the public releases; only the data interface and the input width are adapted.

\paragraph{Conditioning.}  BayesNDE fits a single model
$p_{\vtheta}(x \mid z, y)$ in which the one-hot label is concatenated to the
input of every residual block of the generator, of the encoder, and of both
discriminators; one set of parameters is shared across all labels and is fitted
on the full training split.  The baselines are unconditional density estimators
whose public implementations have no label pathway, so for each of them we fit
one independent unconditional expert per observed label, seeded at
$42 + y$, and report the expert's density as $p(x \mid y)$.  Neither
construction assigns a density to the discrete label itself, and both are
evaluated on the same points.

\paragraph{MADE.}  Masked Autoencoder for Distribution Estimation provides the
autoregressive, single-pass reference point.  We call the official
\texttt{maf-master} entry point \texttt{experiments.train\_made([100,100],
\textquotesingle relu\textquotesingle, \textquotesingle sequential\textquotesingle)},
which builds \texttt{ml.models.mades.GaussianMade}: two hidden layers of $100$
ReLU units, sequential input ordering, no batch normalization, and one Gaussian
conditional per coordinate.  Training uses the authors' \texttt{ml.trainers.SGD}
with Adam at step size $10^{-3}$, weight decay $10^{-6}$, minibatch $100$,
validation monitored every epoch, and early stopping with patience $30$.

\paragraph{MAF.}  Masked Autoregressive Flow stacks MADE blocks to obtain a
deeper autoregressive density.  We call
\texttt{experiments.train\_maf([100,100], \textquotesingle relu\textquotesingle,
10, \textquotesingle sequential\textquotesingle)}, giving
\texttt{ml.models.mafs.MaskedAutoregressiveFlow} with $10$ MADE blocks, each
with two hidden layers of $100$ ReLU units and sequential ordering, and batch
normalization between consecutive blocks.  The trainer is the same as for MADE
with the authors' flow step size $10^{-4}$.

\paragraph{Real NVP.}  Real NVP is the affine-coupling counterpart, whose Jacobian
is triangular by construction rather than by masking order.  We call
\texttt{experiments.train\_Real NVP([100,100], \textquotesingle tanh\textquotesingle,
\textquotesingle relu\textquotesingle, 10)}, giving \texttt{ml.models.nvps.Real NVP}
with $10$ coupling layers, scale and translation networks of two hidden layers
of $100$ units with \texttt{tanh} and ReLU activations respectively, and batch
normalization.  Optimisation settings match MAF.

\paragraph{CFMs.}
Conditional Flow Matching (CFMs) learns a continuous-time velocity field by
regression against a conditional probability path, avoiding ODE simulation
during training \citep{lipman2023flow,tong2024improving}.
We use the TorchCFM implementation from the official
\texttt{conditional-flow-matching} repository of
\citet{tong2024improving}, including its \texttt{CFMLitModule},
conditional flow-matching objective, and \texttt{VelocityNet}
implementation.

\paragraph{Resflow.}  Residual Flows give an invertible free-form architecture
with an unbiased log-determinant estimator, and are the strongest baseline we
find in high dimension.  We use the official \texttt{residual-flows} layers with
$100$ \texttt{iResBlock}s.  Each block is a residual network of Swish
activations and spectrally normalised linear layers with Lipschitz coefficient
$0.9$, $5$ power iterations, and domain/codomain exponents $2.0$; following the
authors' \texttt{train\_toy.py} condition literally, the final layer is
zero-initialised only when its output width is $2$. The log-determinant uses the unbiased geometric Russian-roulette series
(\texttt{n\_dist=\textquotesingle geometric\textquotesingle}, one sample, no
exact trace).  We train with Adam at learning rate $10^{-3}$ and weight decay
$10^{-5}$, batch size $500$, for $50{,}000$ iterations, refreshing the Lipschitz
constants every $5$ training iterations and with $200$ iterations before each
evaluation, validating every $100$ iterations and retaining the
best-validation checkpoint.

\paragraph{Roundtrip.}
Roundtrip is the closest comparator to our setting: it places a latent
variable under a generative network and estimates pointwise density by
one-sided importance sampling.  We use the authors' original TensorFlow
density-estimation implementation and architecture: a generator with $10$
hidden layers of $512$ units, an encoder with $10$ hidden layers of $256$
units, a latent discriminator with two hidden layers of $128$ units, and a
data discriminator with four hidden layers of $256$ units.  Training uses
Adam with learning rate $2\times10^{-4}$, batch size $64$, cycle weights
$\alpha=\beta=10$, and an image pool of size $50$.  Models are trained for
at most $100$ epochs, with validation-based model selection beginning at
epoch $30$ and early-stopping patience $5$. Pointwise densities are estimated using the authors' one-sided importance
sampler with a degree-one Student-$t$ proposal, $40{,}000$ importance samples. Selections of standard deviation and proposal scale follows the original paper.

\paragraph{Execution environments.}  MADE, MAF and Real NVP are the original
Theano~1.0.5 implementations and are executed in a dedicated Python~3.10
environment through a subprocess worker; CFMs and Resflow run under
PyTorch~1.12 on a single GPU; Roundtrip and BayesNDE run under the project
TensorFlow environment.  Because the three environments cannot share a process,
we verified that the evaluation matrices they read are byte-identical by
comparing SHA-256 digests computed by independent, separately audited
implementations of the same hashing routine.

\end{document}